\documentclass[11pt]{article} 
\usepackage{fullpage}
\usepackage{authblk}
\usepackage{algorithm}
\usepackage{algorithmic}
\usepackage{multirow}
\usepackage{booktabs}
\usepackage{pifont} 
\usepackage[utf8]{inputenc} 
\usepackage[T1]{fontenc}    
\usepackage{hyperref}       
\usepackage{url}            
\usepackage{booktabs}       
\usepackage{amsfonts}       
\usepackage{nicefrac}       
\usepackage{microtype}      
\usepackage{amsmath,amsthm}
\usepackage{amssymb,amsbsy,epsfig,float}
\usepackage{graphicx,wrapfig,lipsum}
\usepackage{subfigure} 
\usepackage[makeroom]{cancel}
\usepackage{xspace}
\usepackage{mathtools}

\usepackage{color}
\usepackage[usenames,dvipsnames]{xcolor}

 \usepackage[disable,backgroundcolor = White,textwidth=\marginparwidth]{todonotes}

\newcommand{\todoc}[2][]{\todo[color=Apricot!20,size=\tiny,#1]{Cs: #2}}
\newcommand{\todom}[2][]{\todo[color=Cerulean!20,size=\tiny,#1]{M: #2}}

\newcommand{\hY}{\hat{Y}}
\newcommand{\N}{\mathbb{N}}

\newcommand{\iset}[1]{[#1]}
\DeclareMathOperator{\argmin}{argmin}
\newcommand{\ip}[1]{\langle #1 \rangle} 
\newcommand{\SA}{\mathrm{SA}}
\newcommand{\SD}{\mathrm{SD}}
\newcommand{\WD}{\mathrm{WD}}
\newcommand{\TSA}{\Theta_{\SA}}
\newcommand{\Alg}{\mathfrak{A}}
\newcommand{\TSD}{\Theta_{\SD}}
\newcommand{\TWD}{\Theta_{\WD}}
\newcommand{\awd}{a_{\mathrm{wd}}}
\def\ddefloop#1{\ifx\ddefloop#1\else\ddef{#1}\expandafter\ddefloop\fi}

\def\ddef#1{\expandafter\def\csname b#1\endcsname{\ensuremath{\mathbf{#1}}}}
\ddefloop ABCDEFGHIJKLMNOPQRSTUVWXYZ\ddefloop

\def\ddef#1{\expandafter\def\csname bb#1\endcsname{\ensuremath{\mathbb{#1}}}}
\ddefloop ABCDEFGHIJKLMNOPQRSTUVWXYZ\ddefloop

\def\ddef#1{\expandafter\def\csname c#1\endcsname{\ensuremath{\mathcal{#1}}}}
\ddefloop ABCDEFGHIJKLMNOPQRSTUVWXYZ\ddefloop

\def\ddef#1{\expandafter\def\csname v#1\endcsname{\ensuremath{\boldsymbol{#1}}}}
\ddefloop ABCDEFGHIJKLMNOPQRSTUVWXYZabcdefghijklmnopqrstuvwxyz\ddefloop

\def\ddef#1{\expandafter\def\csname v#1\endcsname{\ensuremath{\boldsymbol{\csname #1\endcsname}}}}
\ddefloop {alpha}{beta}{gamma}{delta}{epsilon}{varepsilon}{zeta}{eta}{theta}{vartheta}{iota}{kappa}{lambda}{mu}{nu}{xi}{pi}{varpi}{rho}{varrho}{sigma}{varsigma}{tau}{upsilon}{phi}{varphi}{chi}{psi}{omega}{Gamma}{Delta}{Theta}{Lambda}{Xi}{Pi}{Sigma}{varSigma}{Upsilon}{Phi}{Psi}{Omega}\ddefloop

\newcommand{\Y}{\mathcal{Y}}
\newcommand{\A}{\mathcal{A}}
\newcommand{\EE}[1]{\mathbb{E}\left[#1\right]}
\newcommand{\EEi}[2]{\mathbb{E}_{#1}\left[#2\right]}
\newcommand{\Prob}[1]{\mathbb{P}\left(#1\right)}
\newcommand{\Regret}{\mathfrak{R}}
\newcommand{\R}{\mathbb{R}} 

\newcommand{\X}{\mathcal{X}}

\usepackage[capitalize]{cleveref}

\newtheorem{thm}{Theorem}

\newtheorem{prop}{Proposition}
\newtheorem{cor}{Corollary}

\newtheorem{defi}{Definition}

\begin{document}

%

%
\title{Sequential Learning without Feedback\thanks{Manjesh Hanawal is with the Department of Industrial Engineering at IIT Bombay. His research was conducted while he was a post-doctoral associate at Boston University. Csaba Szepesvari is with the Department of Computer Science at University of Alberta. Venkatesh Saligrama is with the Department of Electrical and Computer Engineering at Boston University.}}
\author{Manjesh Hanawal \,\,\,\, Csaba Szepesvari \,\,\,\, Venkatesh Saligrama}



\date{}
\maketitle
\begin{abstract}
In many security and healthcare systems a sequence of features/sensors/tests are used for detection and diagnosis. Each test outputs a prediction of the latent state, and carries with it inherent costs. Our objective is to {\it learn} strategies for selecting tests to optimize accuracy \& costs. Unfortunately it is often impossible to acquire-in-situ ground truth annotations and we are left with the problem of unsupervised sensor selection (USS). We pose USS as a version of stochastic partial monitoring problem with an {\it unusual} reward structure (even noisy annotations are unavailable). Unsurprisingly no learner can achieve sublinear regret without further assumptions. To this end we propose the notion of weak-dominance. This is a condition on the joint probability distribution of test outputs and latent state and says that whenever a test is accurate on an example, a later test in the sequence is likely to be accurate as well.
We empirically verify that weak dominance holds on real datasets and prove that it is a maximal condition for achieving sublinear regret. We reduce USS to a special case of multi-armed bandit problem with side information and develop polynomial time algorithms that achieve sublinear regret.
\end{abstract}

\section{Introduction}
Sequential sensor acquisition arises in many security and healthcare diagnostic systems. In these applications we have a diverse collection of sensor-based-tests with differing costs and accuracy. 
In these applications (see Fig.~\ref{motiv}) inexpensive tests are first conducted and based on their outcomes a decision for acquiring more (expensive) tests are made. 
\begin{figure}[t]
  \centering
  \includegraphics[width=0.6\textwidth]{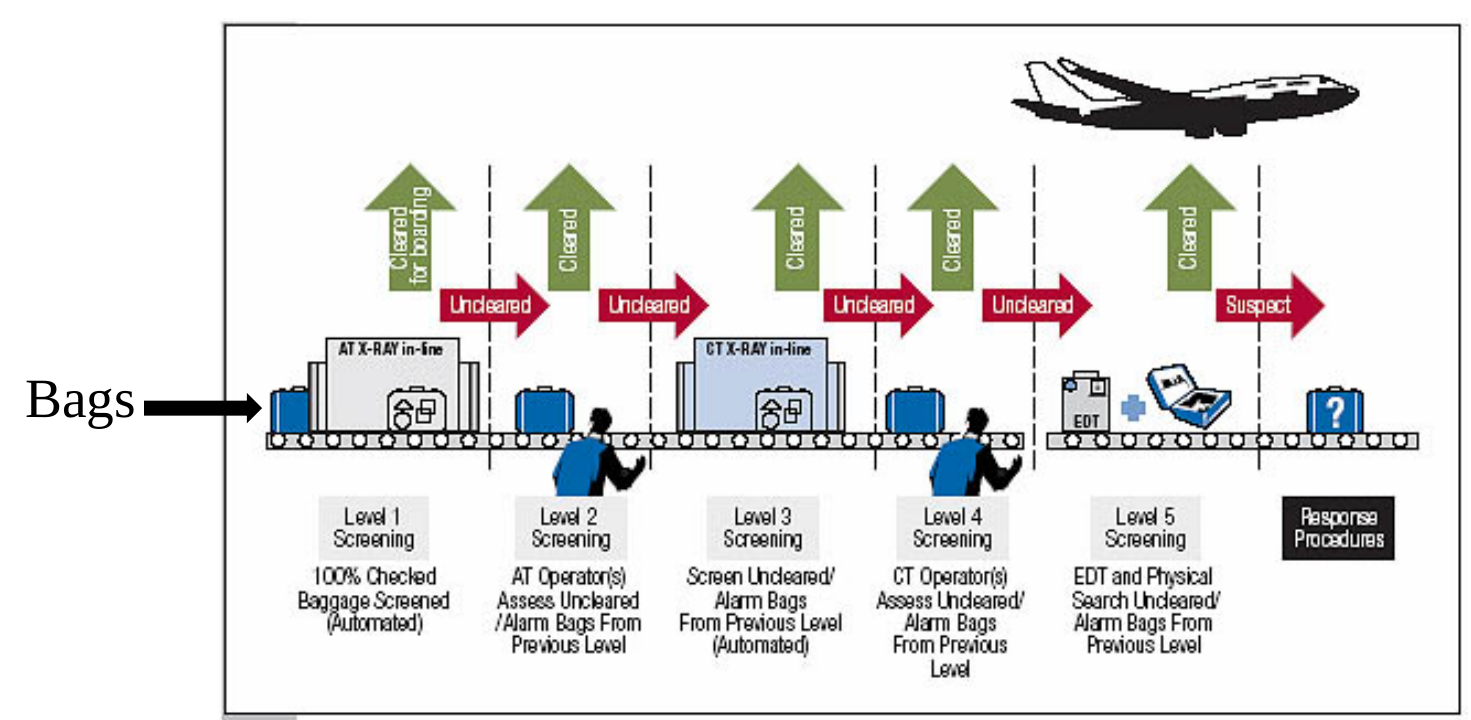}
  \caption{\footnotesize Sequential Sensor/Test Selection in Airport Security Systems. A number of different imaging and non-imaging tests are sequentially processed (see \cite{ML13_MultistageClassifier_TrapezSaligramaCastanon}). Costs can arise due to sensor availability and delay. Inexpensive tests are first conducted and based on their outcomes more expensive tests are conducted.}
  \label{motiv}
\end{figure}
%
The goal in these systems is to maximize overall accuracy for an available cost-budget. Generally, the components that can be optimized include sensor classifiers (to improve test accuracy), sensor ordering, and decision strategies for sequential sensor selection. Nevertheless, sensor classifiers and sensor ordering are typically part of the infrastructure and generally harder to control/modify in both security and medical systems. To this end we focus here  on the sequential sensor selection problem and use the terms sensor and test interchangeably. The need for systematically learning optimal decision strategies for balancing accuracy \& costs arises from the fact that these applications involve legacy systems where sensor/test selection strategies are local; often managed under institutional, rather than national guidelines~\cite{baghdadian}. 
While it is possible to learn such decision strategies given sufficient annotated training data, what makes these applications challenging is that it is often difficult to acquire in-situ ground truth labels. 

These observations leads us to the problem of learning decision strategies for optimal sensor selection in situations where we do not have the benefit of ground-truth annotations, and what we refer to as the Unsupervised Sensor Selection (USS) problem. 
In Sec.~\ref{sec:Setup} we pose our problem as a version of stochastic partial monitoring problem \cite{BaFoPaRaSze14} with \emph{atypical} reward structure, where tests are viewed as actions and sequential observations serves as side information. As is common, we pose the problem in terms of competitive optimality. We consider a competitor who can choose an optimal test with the benefit of hindsight. Our goal is to minimize cummulative regret based on learning the optimal test based on observations in multiple rounds of play. 
Recall that in a stochastic partial monitoring problem a decision maker needs to choose the action with the lowest expected cost by repeatedly trying the actions and observing some feedback.
The decision maker lacks the knowledge of some key information, such as in our case, the misclassification
error rates of the classifiers, but had this information been available, the decision maker could calculate the
expected costs of all the actions (sensor acquisitions) and could choose the best action (test). The feedback received by the decision maker in a given round depends stochastically on the unknown information and the action chosen.
Bandit problems \cite{Tho33} are a special case of partial monitoring, where the key missing information is the expected
cost for each action, and the feedback is the noisy version of the expected cost of the action chosen.
In the USS problem the learner only observes the outputs of the classifiers, but not the label to be predicted over multiple rounds
in a stochastic, stationary environment.

In Sec.~\ref{sec:Setup} we first show that, unsurprisingly that no learner can achieve sublinear regret without further assumptions. 
To this end we propose the notions of weak and strong dominance on tests, which correspond to constraints on joint probability distributions over the latent state and test-outcomes. Strong dominance (SD) is a property arising in many Engineered systems and says that whenever a test is accurate on an example, a later test in the sequence is almost surely accurate on that example. 
Weak dominance is a relaxed notion that allows for errors in these predictions. We empirically demonstrate that weak dominance appears to hold by evaluating it on several real datasets. We also show that in a sense weak dominance is fundamental, namely, without this condition there exist problem instances that result in linear regret. On the other hand whenever this condition is satisfied there exist polynomial time algorithms that lead to sublinear regret. 

Our proof of sublinear regret in Sec.~\ref{sec:Equiv} is based on reducing USS to a version of multi-armed bandit problem (MAB) with side-observation. The latter problem has already been shown to have sub-linear regret in the literature. In our reduction, we identify tests as bandit arms. 
The payoff of an arm is given by marginal losses relative to the root test, and the side observation structure is defined by the feedback graph induced by the directed graph. We then formally show that there is a one-to-one mapping between algorithms for USS and algorithms for MAB with side-observation. In particular, under weak dominance, the regret bounds for MAB with side-observation then imply corresponding regret bounds for USS.

\subsection{Related Work}
%
%
In contrast to our USS setup there exists a wide body of literature dealing with sensor acquisition (see\cite{AISTATS13_SupervisedSequentialLearning_TrapezSaligram,NIPS2015_DirectedAcyclic_WangTrapezSaligram,ICML2015_FeatureBudgeted_NanWangSaligrama}). Like us they also deal with cascade models with costs for features/tests but their method is based on training decision strategies with fully supervised data for prediction-time use. There are also several methods that work in an online bandit setting and train prediction models with feature costs \cite{SBCA14:BanditsPaid} but again they require ground-truth labels as reward-feedback. A somewhat more general version of \cite{SBCA14:BanditsPaid} is developed in \cite{ZBGGySz13:CostlyFeatures} where in addition the learner can choose to acquire ground-truth labels for a cost.

%

Our paper bears some similarity with the concept of active classification, which deals with learning stopping policies\cite{poczos2009,ActiveClass-AIJ-s} among a given sequence of tests. Like us these works also consider costs for utilizing tests and the goal is to learn when to stop to make decisions. Nevertheless, unlike our setup the loss associated with the decision is observed in their context. 
Our paper is related to the framework of finite partial monitoring problems\cite{BaFoPaRaSze14}, which deals with how to infer unknown key information and where tests/actions reveal different types of information about the unknown information. In this context 
\cite{AgTeAn89:pmon} consider special cases where payoff/rewards for a subset of actions are observed. This is further studied as a side-observation problem in \cite{MaSh11} and as graph-structured feedback \cite{COLT15_OnlineLearningWithFeedback_AlonBianchiDekel, NIPS13_FromBanditsToExperts_AlonBianchiGentile,WGySz:NIPS15}). Our work is distinct from these setups because we are unable to observe rewards for our chosen actions or any other actions.
\vspace{-10pt}


\section{Background}
\label{sec:background}

In this section we will introduce a number of sequential decision making problems,
namely stochastic partial monitoring, bandits and bandits with side-observations, which we will build upon later.

First, a few words about our notation: We will use upper case letters to denote random variables.
The set of real numbers is denoted by $\R$. For positive integer $n$, we let
$[n] = \{1,\dots,n\}$. 
We let $M_1(\X)$ to denote the set of probability distributions over some set $\X$.
When $\X$ is finite with a cardinality of $d \doteq |\X|$, 
$M_1(\X)$ can be identified with the $d$-dimensional probability simplex.

%
In a \emph{stochastic partial monitoring problem (SPM)} a learner interacts with a stochastic environment in a sequential manner.
In round $t=1,2,\dots$ the learner chooses an action $A_t$ from an action set $\A$, and receives a feedback $Y_t\in \Y$
from a distribution $p$ which depends on the action chosen and also on the environment instance identified
with a ``parameter'' $\theta\in\Theta$:
$Y_t \sim p(\cdot;A_t,\theta)$. 
The learner also incurs a reward $R_t$, which is a function of the action chosen and the unknown parameter $\theta$:
$R_t = r(A_t,\theta)$. 
The reward may or may not be part of the feedback for round $t$.
The learner's goal is to maximize its total expected reward.
The family of distributions $(p(\cdot;a,\theta))_{a,\theta}$ and the family of rewards $(r(a,\theta))_{a,\theta}$
and the set of possible parameters $\Theta$ are known to the learner, who uses this knowledge to judiciously choose
its next action to reduce its uncertainty about $\theta$ so that it is able to eventually converge on choosing only an 
optimal action $a^*(\theta)$, achieving the best possible reward per round, $r^*(\theta) = \max_{a\in \A} r(a,\theta)$.
The quantification of the learning speed is given by the expected regret 
$\Regret_n = n r^*(\theta) - \EE{\sum_{t=1}^n R_t}$, which, for brevity and when it does not cause confusion, 
we will just call regret.
A sublinear expected regret, i.e., $\Regret_n/n \to 0$ as $n\to \infty$ means that the learner in the long run collects
almost as much reward on expectation as if the optimal action was known to it.
In some cases it is more natural to define the problems in terms of costs as opposed to rewards;
in such cases the definition of regret is modified appropriately by flipping the sign of rewards and costs. 

\emph{Bandit Problems} are a special case of SPMs where 
$\Y$ is the set of real numbers, $r(a,\theta)$ is the mean of distribution $p(\cdot;a,\theta)$ and thus the learner in every round the learner upon choosing an action $A_t$ receives the noisy reward $Y_t \sim p(\cdot;A_t,\theta)$ as feedback. 
A finite armed \emph{bandit with side-observations} \cite{MaSh11} is also a special case of SPMs, where the learner upon choosing an action $a \in \A$ receives noisy reward observations, namely, $Y_t  = (Y_{t,a})_{a\in N(A_t)},\,\,Y_{t,a} \sim p_r(\cdot;a,\theta),\,\,\EE{Y_{t,a}} = r(a,\theta)$, from a neighbor-set $\cN(a) \subset \A$, which is a priori known to the learner. 
%
%
%
To cast this setting as an SPM we let $\Y$ as the set $\cup_{i=0}^K \R^i$ and define 
the family of distributions $(p(\cdot;a,\theta))_{a,\theta}$ such that $Y_t \sim p(\cdot;A_t,\theta)$.
The framework of SPM is quite general and allows for parametric and non-parametric sets $\Theta$. 
In what follows we identify $\Theta$ with set of instances $(p(\cdot;a,\theta),r(a,\theta))_{\theta\in \Theta}$.
In other words we view elements of $\Theta$ as a pair $p,r$ where $p(\cdot;a)$ is a probability distribution over $\Y$ for each $a\in \A$ and $r$ is a map from $\A$ to the reals.

\section{Unsupervised Sensor Selection}
\label{sec:Setup}
\newcommand{\ind}[1]{\mathbb{I}\{#1\}}
%
\vspace{-5pt}
\noindent
{\bf Preliminaries and Notation:} Proofs for formal statements appears in the Appendix. We will use upper case letters to denote random variables.
The set of real numbers is denoted by $\R$. For positive integer $n$, we let
$[n] = \{1,\dots,n\}$. 
We let $M_1(\X)$ to denote the set of probability distributions over some set $\X$.
When $\X$ is finite with a cardinality of $d \doteq |\X|$, 
$M_1(\X)$ can be identified with the $d$-dimensional probability simplex.

We first consider the {\it unsupervised, stochastic, 
cascaded sensor selection} problem. We cast it as a special case of stochastic partial monitoring problem (SPM), which is described in the appendix.
Later we will briefly describe extensions to tree-structures and contextual cases. 
Formally, 
a problem instance is specified by a pair $\theta = (P,c)$, where $P$ is
a distribution over the $K+1$ dimensional hypercube, and $c$ is a $K$-dimensional, nonnegative valued vector
of costs.
While $c$ is known to the learner from the start, $P$ is initially unknown. Henceforth we identify problem instance $\theta$ by $P$. 
The instance parameters specify the learner-environment interaction as follows:
In each round for $t=1,2,\dots$, 
the environment generates a $K+1$-dimensional binary vector
$Y = (Y_t,Y_t^1,\dots,Y_t^K)$ chosen at random from $P$.
Here, $Y_t^i$ is the output of sensor $i$, while $Y_t$ is a (hidden) label to be guessed by the learner.
Simultaneously, the learner chooses an index $I_t\in [K]$ and observes the sensor outputs $Y_t^1,\dots,Y_t^{I_t}$.
The sensors are known to be ordered from least accurate to most accurate, 
i.e., $\gamma_k(\theta) \doteq \Prob{Y_t\ne Y_t^k}$ is decreasing with $k$ increasing.
\begin{figure}[!h]
	\centering
	\includegraphics[scale=.4]{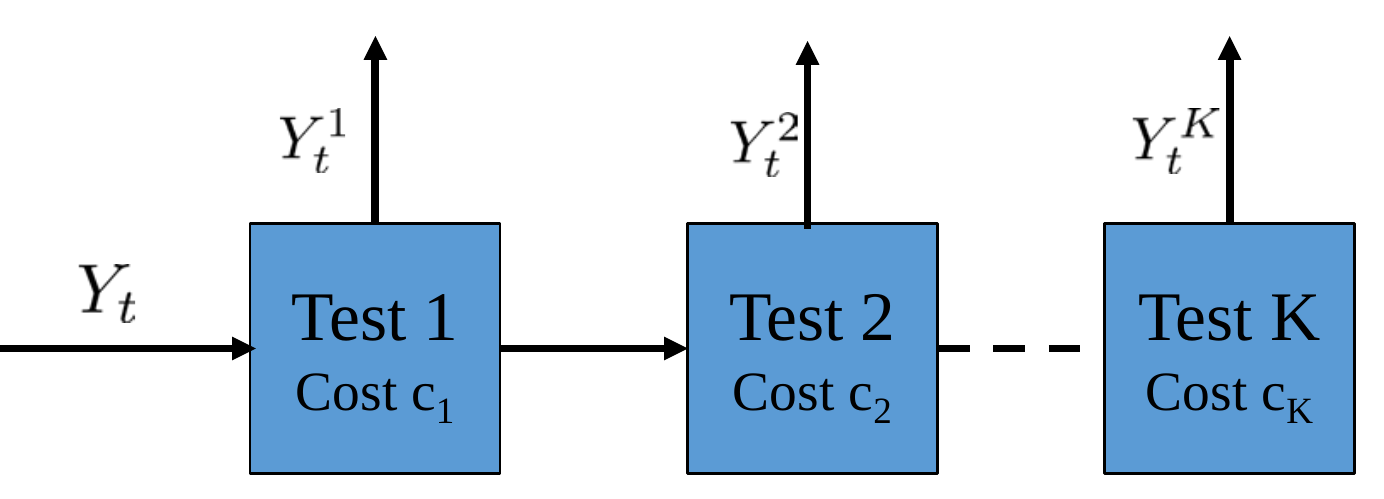}
	\caption{\footnotesize Cascaded Unsupervised Sequential Sensor Selection. $Y_t$ is the hidden state of the instance and $Y_t^1, Y_t^2 \ldots$ are test outputs. Not shown are features that a sensor could process to produce the output. The cost for using test $j$ is also shown.}
	\label{fig:SensorCascade}
	\vspace{-.2cm}
\end{figure} 
\todom{How about this figure and caption?}
Knowing this, the learner's choice of $I_t$ also indicates that he/she chooses $I_t$ to predict the unknown label $Y_t$.
Observing sensors is costly: The cost of choosing $I_t$ is $C_{I_t} \doteq c_1 + \dots + c_{I_t}$.
The total cost suffered by the learner in round $t$ is thus $C_{I_t} + \ind{Y_t \ne Y_t^{I_t}}$.
The goal of the learner is to compete with the best choice given the hindsight of the values $(\gamma_k)_k$.
The expected regret of learner up to the end of round $n$ is 
$\Regret_n =( \sum_{t=1}^n \EE{ C_{I_t} +\ind{Y_t \ne Y_t^{I_t} }} )- n \min_k (C_k + \gamma_k)$.

\noindent
{\bf Sublinear Regret:} The quantification of the learning speed is given by the expected regret 
$\Regret_n = n r^*(\theta) - \EE{\sum_{t=1}^n R_t}$, which, for brevity and when it does not cause confusion, 
we will just call regret. A sublinear expected regret, i.e., $\Regret_n/n \to 0$ as $n\to \infty$ means that the learner in the long run collects almost as much reward on expectation as if the optimal action was known to it.

For future reference, we let $c(k,\theta) = \EE{ C_{k} +\ind{Y_t \ne Y_t^{k} }}  (= C_k + \gamma_k)$ and $c^*(\theta) = \min_k c(k,\theta)$. Thus, $\Regret_n =( \sum_{t=1}^n \EE{ c(I_t,\theta) }) - n c^*(\theta)$.
In what follows, we shall denote by $\cA^*(\theta)$ the set of optimal actions of $\theta$
and we let $a^*(\theta)$ denote the optimal action that has the smallest index
\footnote{Note that even if $i<j$ are optimal actions, there can be suboptimal actions in the interval $[i,j] (=[i,j]\cap \N)$
(e.g., $\gamma_1=0.3$, $C_1=0$, $\gamma_2=0.25$, $C_2=0.1$, $\gamma_3=0$, $C_3=0.3$.}. Thus,
in particular, $a^*(\theta) = \min \cA^*(\theta)$. 
Next, for future reference note that one 
can express optimal actions from the viewpoint of marginal costs and marginal error. 
In particular an action $i$ is optimal if for all $j > i$ the marginal increase in cost, $C_j - C_i$, 
is larger than the marginal decrease in error, $\gamma_i - \gamma_j$: $ \forall \, j \geq i\,$
\begin{equation} \label{eqn:interp_opt}
\vspace{-5pt}
\underbrace{C_j - C_i}_{\text{Marginal Cost}} \geq \underbrace{ E \left [ \ind{Y_t \ne Y_t^i} - \ind{Y_t \ne Y_t^j} \right ]}_{\text{Marginal Error $= \gamma_i-\gamma_j$}}.
\end{equation}

\section{When is USS Learnable?}
\label{sec:Learnability}
Let $\TSA$ be the set of all stochastic, cascaded sensor acquisition problems. 
Thus, $\theta \in \TSA$ such that if $Y\sim \theta$ then $\gamma_k(\theta):=\Prob{Y\ne Y^k}$ 
is a decreasing sequence.
Given a subset $\Theta\subset \TSA$, we say that $\Theta$ is \emph{learnable} 
if there exists a learning algorithm $\Alg$ such that
for any $\theta\in \Theta$, the expected regret $\EE{ \Regret_n(\Alg,\theta) }$ 
of algorithm $\Alg$ on instance $\theta$ is sublinear.
A subset $\Theta$ is said to be a maximal learnable problem class if it is learnable and for any $\Theta'\subset \TSA$ superset
of $\Theta$, $\Theta'$ is not learnable.
In this section we study two special learnable problem classes, $\TSD\subset \TWD$, where the regularity properties of the instances in $\TSD$ are more intuitive, while $\TWD$ can be seen as a maximal extension of $\TSD$.

Let us start with some definitions.
Given an instance $\theta \in \TSA$, we can decompose $\theta$ (or $P$) into the joint distribution $P_S$ of the sensor outputs $S = (Y^1,\dots,Y^k)$ and the conditional distribution of the state of the environment, given the sensor outputs, $P_{Y|S}$.
Specifically, letting $(Y,S)\sim P$, for $s\in \{0,1\}^K$ and $y\in \{0,1\}$, $P_S(s) = \Prob{S = s}$ and $P_{Y|S}(y|s) = \Prob{Y=y|S=s}$. We denote this by $P = P_S \otimes P_{Y|S}$.
A learner who observes the output of all sensors for long enough is able to identify 
$P_S$ with arbitrary precision, while $P_{Y|S}$ remains hidden from the learner.
 This leads to the following statement:
\begin{prop}
\label{prop:learnablemap}
A subset $\Theta\subset \TSA$ is learnable 
if and only if there exists a map $a: M_1( \{0,1\}^K )\to [K]$ such that 
for any $\theta \in \Theta$ 
with decomposition $P = P_S \otimes P_{Y|S}$, $a(P_S)$ is an optimal action in $\theta$.
\end{prop}

An action selection map  $a: M_1( \{0,1\}^K ) \to [K]$ is said to be \emph{sound} for an instance 
$\theta\in \TSA$ with $\theta = P_S\otimes P_{Y|S}$ if $a(P_S)$ selects an optimal action in $\theta$.
With this terminology, the previous proposition says that a set of instances $\Theta$ is learnable if and only if there exists a
sound action selection map for all the instances in $\Theta$.

A class of sensor acquisition problems that contains instances that satisfy the so-called \emph{strong dominance} condition 
will be shown to be learnable:
\begin{defi}[Strong Dominance (SD)]
	An instance $\theta \in \TSA$  is said to satisfy the \emph{strong dominance property} if 
	it holds in the instance that if a sensor predicts correctly
	then all the sensors in the subsequent stages of the cascade also predict correctly, i.e., 
	for any $i\in [K]$,
	\begin{equation}
	\label{eqn:DominanceCondition}
	Y^i=Y \,\, \Rightarrow\,\, Y^{i+1}= \dots =  Y^K = Y
	\end{equation}
	almost surely (a.s.)
	where $(Y,Y^1,\dots,Y^K)\sim P$.
\end{defi}
Before we develop this concept further we will motivate strong dominance based on experiments on a few real-world datasets. SD property naturally arises in the context of error-correcting codes (see \cite{costello,voyager}), where one sequentially corrects for errors, and thus whenever a decoder is accurate the following decoders are also accurate. On the other hand for real-datasets SD holds only approximately. Table \ref{tab:ErrorTable1} lists the error probabilities of the classifiers (sensors) for the heart and diabetic datasets from UCI repository. We split features into two sets based on provided costs (cheap tests are based on patient history and costly tests include all the features). We then trained an SVM classifier with 5-fold cross-validation and report scores based on held-out test data. 
\begin{table}[h]
\vspace{-6pt}
\begin{center}
\begin{tabular}[c]{c|c|c|c } 
Dataset & $\gamma_1$ & $\gamma_2$ & $\delta_{12}$\\ \hline \hline
PIMA Diabetes & $0.32 $ & $ 0.23$  & 0.065\\  \hline
Heart (Cleveland) & $0.305$ & $0.169$ &  0.051\\  \hline
\end{tabular}
\label{tab:ErrorTable1}
\caption{\footnotesize Depicts approximate Strong Dominance on real datasets: $\gamma_1 \triangleq \Pr(Y^1 \neq Y),\,\gamma_2\triangleq \Pr(Y^2 \neq Y),\, \delta_{12} \triangleq \Pr(Y^1=Y,\, Y^2\neq Y)$ }
\end{center}
\vspace{-13pt}
\end{table}
The last column lists the probability that second sensor misclassifies an instance that is correctly classified by the first sensor. SD is the notion that suggests that this probability is zero. We find in these datasets that $\delta_{12}$ is small thus justifying our notion. In general we have found this behavior is representative of other cost-associated datasets. 

We next show that strong dominance conditions ensures learnability. To this end,
let $\TSD = \{ \theta\in \TSA\,:\, \theta \text{ satisfies the strong dominance condition } \}$.

\begin{thm}
\label{thm:tsdlearnable}
The set $\TSD$ is learnable.
\end{thm}
We start with a proposition that will be useful beyond the proof of this result.
In this proposition, $\gamma_i = \gamma_i(\theta)$ for $\theta \in \TSA$ and $(Y,Y^1,\dots,Y^K) \sim \theta$.
\begin{prop}\label{prop:gammadiff}
For any $i,j\in [K]$, $\gamma_i - \gamma_j = \Prob{Y^i\ne Y^j} - 2\Prob{ Y^j \ne Y, Y^i=Y}$.
\end{prop}

The proof motivates the definition of weak dominance, a concept that we develop next through a series of smaller
propositions. In these propositions, as before $(Y,Y^1,\dots,Y^K) \sim P$ where $P\in M_1(\{0,1\}^{K+1})$,
 $\gamma_i = \Prob{Y^i \ne Y}$, $i\in [K]$, and $C_i = c_1 + \cdots + c_i$.
We start with a corollary of \cref{prop:gammadiff}
\begin{cor}
\label{cor:gammadiff}
Let $i<j$. Then $0\le \gamma_i -\gamma_j \le \Prob{Y^i\ne Y^j}$.
\end{cor}
\begin{prop}
\label{prop:ilej}
Let $i<j$. Assume 
\begin{align}
\label{eq:cond1}
C_j - C_i \not\in [\gamma_i - \gamma_j, \Prob{Y^i\ne Y^j} )\,.
\end{align}
Then $\gamma_i + C_i \le \gamma_j + C_j$ if and only if $C_j - C_i \ge \Prob{Y^i\ne Y^j}$.
\end{prop}
\begin{proof}
\noindent $\Rightarrow$: From the premise, it follows that $C_j - C_i \ge \gamma_i - \gamma_j$.
Thus, by~\eqref{eq:cond1}, $C_j - C_i \ge \Prob{Y^i\ne Y^j}$.
\noindent $\Leftarrow$: We have $C_j - C_i \ge \Prob{Y^i \ne Y^j} \ge \gamma_i -\gamma_j$, where the last
inequality is by \cref{cor:gammadiff}.
\end{proof}
\begin{prop}
\label{prop:jlei}
Let $j<i$. Assume
\begin{align}
\label{eq:cond2}
C_i - C_j \not\in (\gamma_j - \gamma_i, \Prob{Y^i \ne Y^j} ]\,.
\end{align}
Then, $\gamma_i + C_i \le \gamma_j + C_j$ if and only if $C_i - C_j \le \Prob{Y^i \ne Y^j}$.
\end{prop}
\begin{proof}
\noindent $\Rightarrow$: The condition $\gamma_i + C_i \le \gamma_j + C_j$ implies that $\gamma_j -\gamma_i \ge C_i - C_j$.
By \cref{cor:gammadiff} we get $\Prob{Y^i \ne Y^j} \ge C_i - C_j$.
\noindent $\Leftarrow$: Let $C_i - C_j \le \Prob{Y^i \ne Y^j}$. Then, by \eqref{eq:cond2}, $C_i - C_j \le \gamma_j - \gamma_i$.
\end{proof}
These results motivate the following definition:
\begin{defi}[Weak Dominance]
	An instance $\theta \in \TSA$  is said to satisfy the \emph{weak dominance property} if 
	for $i = a^*(\theta)$,
	\begin{align}
	\label{eq:wd} \forall j>i\,\,: \,\, C_j - C_i \ge \Prob{Y^i\ne Y^j}\,.
	\end{align}
We denote the set of all instances in $\TSA$ that satisfies this condition by $\TWD$.	
\end{defi}
Note that $\TSD\subset \TWD$ since for any $\theta\in \TSD$, any $j>i = a^*(\theta)$, on the one hand $C_j - C_i \ge \gamma_i - \gamma_j$, while on the other hand, by the strong dominance property, $\Prob{Y^i\ne Y^j} = \gamma_i - \gamma_j$.

We now relate weak dominance to the optimality condition described in Eq.~\eqref{eqn:interp_opt}. Weak dominance can be viewed as a more stringent condition for optimal actions. Namely, for an action to be optimal we also require that the marginal cost be larger than marginal \emph{absolute} error:
\begin{equation} \label{eqn:interp_WD}
\underbrace{C_j - C_i}_{\text{Marginal Cost}} \geq \underbrace{ E \left [ \left | \ind{Y_t \ne Y_t^i} - \ind{Y_t \ne Y_t^j} \right | \right ]}_{\text{Marginal Absolute Error}},\,\,\, \forall \, j \geq i\,.
\end{equation}
The difference between marginal error in Eq.~\eqref{eqn:interp_opt} and marginal absolute error is the presence of the absolute value. We will show later that weak-dominant set is a maximal learnable set, namely, the set cannot be expanded while ensuring learnability.

We propose the following action selector $\awd: M_1(\{0,1\}^K)  \to [K]$:
\begin{defi}\label{def:awd}
For $P_S \in M_1(\{0,1\}^K) $ let $\awd(P_S)$ denote the smallest index $i\in [K]$ such that
\begin{subequations}
\begin{align}
\forall j<i \,\,:\,\, C_i - C_j < \Prob{ Y^i \ne Y^j }\,, \label{eq:wd1}\\ 
\forall j>i \,\,:\,\, C_j - C_i \ge \Prob{ Y^i \ne Y^j }\,, \label{eq:wd2}
\end{align}
\end{subequations}
where $C_i = c_1+\cdots + c_i$, $i\in [K]$ and $(Y^1,\dots,Y^K) \sim P_S$.
(If no such index exists, $\awd$ is undefined, i.e., $\awd$ is a partial function.)
\end{defi}
\begin{prop}
\label{prop:awdwelldef}
For any $\theta \in \TWD$ with $\theta = P_S\otimes P_{Y|S}$, $\awd(P_S)$ is well-defined.
\end{prop}

\begin{prop}
\label{prop:awdsound}
The map $\awd$ is sound over $\TWD$: In particular, for any
$\theta\in \TWD$ with $\theta = P_S\otimes P_{Y|S}$, $\awd(P_S)= a^*(\theta)$.
\end{prop}

\begin{cor}\label{cor:twdlearnable}
The set $\TWD$ is learnable.
\end{cor}
\begin{proof}
By \cref{prop:awdwelldef}, $\awd$ is well-defined over $\TWD$, while by \cref{prop:awdsound}, $\awd$ is sound over $\TWD$.
By \cref{prop:learnablemap}, $\TWD$ is learnable, as witnessed by $\awd$. \todoc{We should add definitions for these concepts..
namely, $\awd$ well-defined over $\TWD$, $\awd$ sound over $\TWD$, etc.}\todom[]{Csaba, you did this already!!}
\end{proof}
\begin{prop}
\label{prop:awdcorrectimplieswd}
Let $\theta \in \TSA$ and $\theta = P_S\otimes P_{Y|S}$ be such that $\awd$ is defined for $P_S$
and $\awd(P_S) = a^*(\theta)$. Then $\theta \in \TWD$.
\end{prop}
\begin{proof}
Immediate from the definitions.
\end{proof}
An immediate corollary of the previous proposition is as follows:
\begin{cor}\label{cor:awdoutsideincorrect}
Let $\theta \in \TSA$ and $\theta = P_S \otimes P_{Y|S}$. 
Assume that $\awd$ is defined for $P_S$ and $\theta \not\in \TWD$. Then $\awd(P_S) \ne a^*(\theta)$.
\end{cor}
The next proposition states that $\awd$ is essentially the only sound action selector map defined for
 all instances derived from instances of $\TWD$:
\begin{prop}\label{prop:awdunique}
Take any action selector map $a: M_1( \{0,1\}^K ) \to [K]$ which is sound over $\TWD$.
Then, for any $P_S$ such that $\theta = P_S\otimes P_{Y|S}\in \TWD$ with some $P_{Y|S}$,
 $a(P_S) = \awd(P_S)$.
\end{prop}

The next result shows that
the set $\TWD$ is essentially a maximal learnable set in $\mathrm{dom}(\awd)$:
\begin{thm}
Let $a: M_1(\{0,1\}^K) \to [K]$ be an action selector map
such that $a$ is sound over the instances of $\TWD$.
Then there is no instance $\theta = P_S\otimes P_{Y|S} \in \TSA\setminus \TWD$ such that 
$P_S\in \mathrm{dom}(\awd)$, the optimal action of $\theta$ is unique
\todoc{It would be nice to remove this uniqueness assumption, but I don't see how this could be made to work.}
 and $a(P_S) = a^*(\theta)$.
\end{thm}
Note that $\mathrm{dom}(\awd)\setminus \{ P_S \,:\, \exists P_{Y|S} \textrm{ s.t. } P_S \otimes P_{Y|S} \in \TWD \} \ne \emptyset$, i.e., the theorem statement is non-vacuous.
In particular, for $K=2$, consider $(Y,Y^1,Y^2)$ such that $Y$ and $Y^1$ are independent and $Y^2 = 1-Y^1$, we can see that the resulting instance gives rise to $P_S$ which is in the domain of $\awd$ for any $c\in \R_+^K$ (because here $\gamma_1 = \gamma_2 = 1/2$, thus $\gamma_1 - \gamma_2 = 0$ while $\Prob{Y^1\ne Y^2}=1$).
\begin{proof}
Let $a$ as in the theorem statement. By~\cref{prop:awdunique}, $\awd$ is the unique sound action-selector map over $\TWD$.
Thus, for any $\theta = P_S\otimes P_{Y|S}\in \TWD$, $\awd(P_S) = a(P_S)$.
Hence, the result follows from \cref{cor:awdoutsideincorrect}.
\end{proof}
While $\TWD$ is learnable, it is not uniformly learnable, i.e., the minimax regret $\Regret_n^*(\TWD) = \inf_{\Alg} \sup_{\theta\in \TWD} \Regret_n(\Alg,\theta)$ over $\TWD$ grows linearly:
\begin{thm}
\label{thm:nonunif}
$\TWD$ is not uniformly learnable:
$\Regret_n^*(\TWD) = \Omega(n)$.
\end{thm}

\noindent
{\bf Extensions:} We describe a few extensions here but omit details due to lack of space.

\noindent
{\it Tree-Architectures:}
We can deal with trees in an analogous fashion. Like for cascades we keep track of disagreements along each path. Under SD these disagreements in turn approximate marginal error. This fact is sufficient to identify the optimal sensor (see Eqn.~\ref{eqn:interp_opt}). Similarly our results also generalize to trees under WD condition. 

\noindent
{\it Context-Dependent Sensor Selection:}
Each example has a context $x \in \mathbb{R}^d$. A test is a mapping that takes (partial) context and maps to an output space. We briefly describe how to extend to context-dependent sensor selection. Analogous to the context-independent case we impose context dependent notions for SD and WD, namely, Eq.~\ref{eqn:DominanceCondition} and Eq.~\ref{eq:wd} hold conditioned on each $x \in {\cal X}$. To handle these cases we let $\gamma_i(x)\triangleq \Pr(Y^{i} \neq Y)$ and $\gamma_{ij}(x) \triangleq \Pr(Y^{i} \neq Y^{j})$ denote the conditional error probability for the ith sensor and the conditional marginal error probability between sensor $i$ and $j$ respectively. Then the results can be generalized under a parameterized GLM model for disagreement, namely, the log-odds ratio $\log \frac{\gamma_{ij}(x)}{1-\gamma_{ij}(x)}$ can be written as $\theta_{ij}'x$ for some unknown $\theta_{ij} \in \mathbb{R}^d$.


\section{Regret Equivalence}
\label{sec:Equiv}
\newcommand{\one}[1]{\mathbb{I}_{\{#1\}}}
\newcommand{\Pside}{\P_{\mathrm{side}}\xspace}
\newcommand{\PSAP}{\P_{\mathrm{SAP}}\xspace}
\newcommand{\PUSS}{\P_{\mathrm{USS}}\xspace}
\renewcommand{\P}{\mathcal{P}}

In this section we establish that USS with SD property is `regret equivalent' to an instance of multi-armed-bandit (MAB) with side-information. The corresponding MAB algorithm can then be suitably imported to solve USS efficiently.   

Let $\PUSS$ be the set of USSs with action set $\A = [K]$.
The corresponding bandit problems will have the same action set,
while for action $k\in [K]$ the neighborhood set is $\mathcal{N}(k) = [k]$.
Take any instance $(P,c)\in \PUSS$ and let $(Y,Y^1,\dots,Y^K) \sim P$ be the 
unobserved state of environment.
We let the reward distribution for arm $k$ in the corresponding bandit problem
be a shifted Bernoulli distribution.
In particular, the cost of arm $k$ follows the distribution of $\one{Y^k\ne Y^1} - C_k$ (we use costs here to avoid flipping signs).

The costs for different arms are defined to be independent of each other.
Let $\Pside$ denote the set of resulting bandit problems and let $f:\PUSS \to \Pside$
be the map that transforms USS instances to bandit instances by following the
transformation that was just described.

Now let $\pi\in \Pi(\Pside)$ be a policy for $\Pside$.
Policy $\pi$ can also be used on any $(P,c)$ instance in $\PUSS$ in an obvious way:
In particular, given the history of actions and states $A_1,U_1,\dots,A_t,U_t$
in $\theta=(P,c)$ where $U_s = (Y_s,Y_s^1,\dots,Y_s^{K})$ such that 
the distribution of $U_s$ given that $A_s=a$ is $P$ marginalized to $\Y^{a}$,
the next action to be taken is 
$A_{t+1}\sim \pi(\cdot| A_1, V_1,\dots,A_t,V_t)$, where 
$V_s = (\one{Y_s^1\ne Y_s^1}-C_1,\dots,\one{Y_s^1\ne Y_s^{A_s}}-C_{A_s})$. Let the resulting policy be denoted by $\pi'$.
The following can be checked by simple direct calculation:
\begin{prop} 
	\label{prop:equivalence}
If $\theta \in \TSD$, then the regret of $\pi$ on $f(\theta)\in \Pside$
is the same as the regret of $\pi^\prime$ on $\theta$. 
\end{prop}

This implies that $\Regret_T^*(\TSD)\le \Regret_T^*(f(\TSD))$. Now note that this reasoning can also be repeated in the other ``direction'': 
For this, first note that the map $f$ has a right inverse $g$ 
(thus, $f\circ g$ is the identity over $\Pside$)
and if $\pi'$ is a policy for $\PUSS$, 
then $\pi'$ can be ``used''  on any instance $\theta\in \Pside$
via the ``inverse'' of the above policy-transformation:
Given the sequence $(A_1,V_1,\dots,A_t,V_t)$ where 
$V_s= (B_s^1+C_1,\dots,B_ s^{K}+C_s)$ is the vector of costs for round $s$
with $B_s^k$ being a Bernoulli with parameter $\gamma_k$,
let $A_{t+1} \sim \pi'( \cdot| A_1,W_1,\dots,A_t,W_t)$ where
$W_s = (B_s^1,\dots,B_s^{A_s})$. Let the resulting policy be denoted by $\pi$.
Then the following holds:
\begin{prop}
Let $\theta \in f(\TSD)$. Then the regret of policy $\pi$ on $\theta\in f(\TSD)$ is the same
as the regret of policy $\pi'$ on instance $f^{-1}(\theta)$.
\end{prop}
Hence, $\Regret_T^*(f(\TSD))\le \Regret_T^*(\TSD)$.
In summary, we get the following result:
\begin{cor}
$\Regret_T^*(\TSD) = \Regret_T^*(f(\TSD))$. 
\end{cor}
{\bf Lower Bounds:} Note that as a consequence of the reduction and the one-to-one correspondence between the two problems, lower bounds for MAB with side-information is a lower bound for USS problem.

\section{Algorithms}
\label{sec:Algo}
\newcommand{\set}{\leftarrow}
\newcommand{\hgamma}{\hat{\gamma}}
The reduction of the previous section suggests that one can  utilize 
an algorithm developed for stochastic bandits with side-observation to solve a USS satisfying SD property.
In this paper we make use of Algorithm~1 of \cite{WGySz:NIPS15}
that was proposed for stochastic bandits with Gaussian side observations. 
As noted in the same paper, the algorithm is also suitable for problems where the payoff distributions are sub-Gaussian.
As Bernoulli random variables are $\sigma^2=1/4$-sub-Gaussian (after centering),
the algorithm is also applicable in our case.

For the convenience of the reader, we give the algorithm resulting from applying the reduction to Algorithm~1 
of \cite{WGySz:NIPS15} in an explicit form.
For specifying the algorithm we need some extra notation.
Recall that given a USS instance $\theta = (P,c)$, we let $\gamma_k = \Prob{Y\ne Y^k}$ where $(Y,Y^1,\dots,Y^K)\sim P$ and $k\in [K]$. Let $k^*=\arg\min_k \gamma_k +C_k$ denote the optimal action and $\Delta_k(\theta)=\gamma_k+C_k-(\gamma_{k^*}+C_{k^*})$ the sub-optimality gap of arm $k$. Further, let $\Delta^*(\theta) = \min\{\Delta_k(\theta), k\neq k^* \}$ denote the smallest positive sub-optimality gap and define $\Delta_k^*(\theta) =\max\{\Delta_k(\theta), \Delta^*(\theta)\}$.

Since cost vector $c$ is fixed, in the following we use parameter $\gamma$ in place of $\theta$ to denote the problem instance.
A (fractional) allocation count $u\in [0,\infty)^K$ determines for each action $i$ how many times the
action is selected.
Thanks to the cascade structure, using an action $i$ implies observing the output of all the sensors with index $j$ less than equal to $i$. Hence, a sensor $j$ gets observed $u_j+u_{j+1}+\dots+u_K$ times.
We call an allocation count ``sufficiently informative'' if (with some level of confidence)
it holds that {\em (i)} for each suboptimal choice, 
the number of observations for the corresponding sensor is sufficiently large to distinguish
it from the optimal choice; and  {\em (ii)}  the optimal choice is also distinguishable from the second best choice.
We collect these counts into the set $C(\gamma)$ for a given parameter $\gamma$:
$C(\gamma) = \{ u\in [0,\infty)^K\,:\, 
u_j+u_{j+1}+\dots+u_K
\ge \frac{2\sigma^2}{(\Delta_j^*(\theta))^2}, j\in [K] \}$
(note that $\sigma^2=1/4$).
Further, let $u^*(\gamma)$
be the allocation count that minimizes the total expected excess cost over the set of sufficiently informative allocation counts:
In particular,  we let $u^*(\gamma) = \argmin_{u\in C(\gamma)} \ip{ u, \Delta(\theta) }$ 
with the understanding that for any optimal action $k$, $u_k^*(\gamma) = \min \{ u_k \,: u\in C(\gamma) \}$ (here, $\ip{x,y} = \sum_i x_i y_i$ is the standard inner product of vectors $x,y$).
For an allocation count $u\in [0,\infty)^K$ let $m(u)\in \N^K$ denote total sensor observations, where $m_j(u) = \sum_{i=1}^j u_i$ corresponds to observations of sensor $j$.

\begin{center}
	\begin{minipage}{0.48\textwidth}
		\begin{algorithm}[H]
			\caption{Algorithm for USS under SD property} 
			\label{alg:asym}
			\begin{algorithmic}[1]
				\STATE Play action $K$ and observe  $Y^1,\dots,Y^K$.
				\STATE Set $\hgamma_{i}^1 \set \one{Y^1\ne Y^i}$ for all $i\in [K]$.
				\STATE Initialize the exploration count: $n_e \set 0$.
				\STATE Initialize the allocation counts: $N_K(1) \set 1$.
				\FOR{$t=2,3,...$}
				\IF{$\frac{N(t-1)}{4\alpha \log t}\in C(\hgamma^{t-1})$} \label{alg:check}
				\STATE Set $I_t \set \argmin_{k\in [K]} c(k,\hgamma^{t-1})$. \label{alg:greedy}
				\ELSE
				\IF{$N_K(t-1)<\beta(n_e)/K$} \label{alg:starve}
				\STATE Set $I_t =K$. \label{alg:forced}
				\ELSE
				\STATE Set $I_t$ to some $i$ for which \label{alg:plan} \\
				$\quad$ $N_i(t-1)< u_i^*(\hgamma^{t-1})4\alpha\log t$.
				\ENDIF
				\STATE Increment exploration count: $n_e \set n_e+1$.
				\ENDIF
				\STATE Play $I_t$ and observe  $Y^1,\dots,Y^{I_t}$.
				\STATE For $i\in [I_t]$, set\\
				$\quad$ $\hgamma_i^t \set (1-{1\over t}) \hgamma_i^{t-1} + {1\over t} \,\one{Y^1\ne Y^i}$.
				\ENDFOR
			\end{algorithmic}
		\end{algorithm}
	\end{minipage}
\end{center}

The idea of \cref{alg:asym} is as follows:
The algorithm keeps track of an estimate $\hgamma^{t}:=(\hgamma^t_i)_{i \in [K]}$ of $\gamma$ in each round, which is initialized by pulling arm $K$ as this arm
gives information about all the other arms.
In each round, the algorithm first checks whether given the current estimate $\hgamma^t$ and the current confidence level (where the confidence level is gradually increased over time), the current allocation count $N(t)\in \N^K$
is sufficiently informative (cf. line \ref{alg:check}). If this holds, the action that is optimal under $\hgamma(t)$ is chosen 
(cf. line \ref{alg:greedy}). If the check fails, we need to explore.
The idea of the exploration is that it tries to ensure that the ``optimal plan'' -- assuming $\hgamma$ is the ``correct'' parameter -- is followed (line \ref{alg:plan}). However, this is only reasonable, if all components of $\gamma$ are relatively well-estimated.
Thus, first the algorithm checks whether any of the components of $\gamma$ has a chance of being
extremely poorly estimated (line \ref{alg:starve}). Note that the requirement here is that a significant, but still altogether diminishing fraction of the \emph{exploration rounds} is spent on estimating each components: In the long run, the fraction of exploration rounds amongst all rounds itself is diminishing; hence the forced exploration of line \ref{alg:forced} overall has a small impact on the regret, while it allows to stabilize the algorithm.

\newcommand{\gap}{d}
\newcommand{\norm}[1]{\|#1\|}
For $\theta \in \TSD$, let $\gamma(\theta)$ be the error probabilities for the various sensors.
The following result follows from Theorem~6 of \cite{WGySz:NIPS15}:
\begin{thm} \label{thm:ftregret}
Let $\epsilon>0$, $\alpha>2$ arbitrary and choose any non-decreasing $\beta(n)$ that satisfies $0\le \beta(n)\le n/2$ and $\beta(m+n)\le \beta(m)+\beta(n)$ for $m,n\in \N$.
Then, 
for any
$\theta\in \TSD$, letting $\gamma = \gamma(\theta)$
the expected regret of \cref{alg:asym} after $T$ steps satisfies 
\begin{align*}
&R_T(\theta)
  \le  \Big( 2K+2+\frac{4K}{\alpha-2} \Big) 
  +  4K\sum_{s=0}^T \exp{\Big( \frac{-8\beta(s)\epsilon^2}{2K} \Big)}~~~~~ \\
 &\hspace{-.3cm} + 2 \beta\Big( 4\alpha\log T\sum_{i\in \iset{K}} u_i^*(\gamma,\epsilon)+K \Big)
  +  4\alpha\log T \sum_{i\in \iset{K}} u_i^*(\gamma,\epsilon)\Delta_i(\theta) \,,
\end{align*}
where $u_i^*(\gamma,\epsilon) = \sup\{u_i^*(\gamma')\,:\, \norm{\gamma'-\gamma}_{\infty}\le \epsilon \}$.
\label{thm:alg1-ub1}
\end{thm}  
Further specifying $\beta(n)$ and using the continuity of $u^*(\cdot)$ at $\theta$, it immediately follows that Algorithm~\ref{alg:asym} achieves asymptotically optimal performance: 
\begin{cor}
\label{cor:alg1-asym-opt}
 Suppose the conditions of Theorem~\ref{thm:alg1-ub1} hold. Assume, furthermore, that $\beta(n)$ satisfies $\beta(n) = o(n)$ and $\sum_{s=0}^\infty \exp \left( -\frac{\beta(s)\epsilon^2}{2K\sigma^2} \right)<\infty$ for any $\epsilon>0$, then for any $\theta$ such that $u^*(\theta)$ is unique, 
$
\limsup_{T\rightarrow \infty} R_T(\theta,c)/\log T \le 4\alpha \inf_{u\in C_\theta} \ip{u,\Delta(\theta)}\,$.
\end{cor}
Note that any $\beta(n) = an^b$ with $a\in (0,\frac{1}{2}]$, $b\in (0,1)$ satisfies the requirements in Theorem~\ref{thm:alg1-ub1} and Corollary~\ref{cor:alg1-asym-opt}.

The algorithm in \ref{alg:asym} only estimates the disagreements $\P\{Y^1 \neq Y^j\}$ for all $j \in [K]$ which suffices to identify the optimal arm when SD property (see Sec \ref{sec:Equiv}) holds. Clearly, one can estimate pairwise disagreements probabilities $\P\{Y^i \neq Y^j\}$ for $i\neq j$ and use them to order the arms. We next develop a heuristic algorithm that uses this information and works for USS under WD. 

\subsection{Algorithm for Weak Dominance}
The reduction scheme described above is optimal under SD property and can fail under the more relaxed WD property. This is because, under SD, marginal error is equal to marginal disagreement (see Prop.~2); it follows for any two sensors $i < j$, $\gamma_i-\gamma_j = (\gamma_i-\gamma_1) - (\gamma_j-\gamma_1) = \Prob{Y^i\ne Y^1} - \Prob{Y^i\ne Y^1}$; and so we can compute marginal error between any two sensors by only keeping track of disagreements relative to sensor 1. Under WD, marginal error is a lower bound and keeping track only of disagreements relative to sensor 1 no longer suffices because $\Prob{Y^i\ne Y^1} - \Prob{Y^i\ne Y^1}$ is no longer a good estimate for $\Prob{Y^i\ne Y^j}$.
%
Our key insight is based on the fact that,
under WD, for a given set of the disagreement probabilities, for all $i \neq j$, the set $\{i \in [K-1]: C_j - C_i \geq \P\{Y^i \neq Y^j\} \mbox{ for all } j>i\}$ includes the optimal arm. We use this idea in Algorithm \ref{alg:UCB} to identify the optimal arm when an instance of USS satisfies WD. We will experimentally validate its performance on real datasets in the next section.
\begin{center}
\begin{minipage}{0.48\textwidth}
		\begin{algorithm}[H]
			\caption{Algorithm for USS with WD property} 
			\label{alg:UCB}
			\begin{algorithmic}[1]
				\STATE Play action $K$ and observe $Y^1,\dots,Y^K$
				\STATE Set $\hgamma_{ij}^1 \set \one{Y^i\ne Y^j}$ for all $i,j\in [K]$ and $i < j$.
				\STATE $n_i(1)\leftarrow \one{i=K} \; \forall i\in [K]$.
				\FOR{$t=2,3,...$}
				\STATE $U_{ij}^t=\hgamma_{ij}^{t-1} + \sqrt{\frac{1.5\log(t)}{n_j(t-1)}}$  $\;\forall \; i,j \in [K]$ and $i<j$ \label{algo:UCB}
				\STATE $S_t=\{i \in [K-1]: C_j-C_i \geq U_{ij}^t \;\forall \;   j > i \}$ \label{algo:sort}
				\STATE Set $I_t= \arg \min S_t \cup \{K\} $
				\STATE Play $I_t$ and observe $Y^1,\dots,Y^{I_t}$.
				 \FOR {$i\in [I_t]$}
				\STATE $n_i(t) \set n_i(t-1)+1$\\
				 \STATE $\hgamma_{ij}^t \set \left (1-\frac{1}{n_j(t)}\right )
				 \hgamma_{ij}^{t-1} + \frac{1}{n_j(t)} \,\one{Y^j\ne Y^i} \forall \; i<j \leq I_t$ \label{algo:Update}
				\ENDFOR
				\ENDFOR
			\end{algorithmic}
		\end{algorithm}
	\end{minipage}
\end{center}

The algorithm works as follows. At each round, $t$, based on history, we keep track of estimates, $\hgamma^t_{ij}$, of disagreements between sensor $i$ and sensor $j$.  
In the first round, the algorithm plays arm $K$ and initializes its values. In each subsequent round, the algorithm computes the upper confidence value of $\hgamma^t_{ij}$ denoted as $U^t_{ij}$ (\ref{algo:UCB}) for all pairs $(i,j)$ and orders the arms: $i$ is considered better than arm $j$ if $C_j-C_i \geq U^t_{ij}$. Specifically, the algorithm plays an arm $i$ that satisfies $C_j-C_i \geq U^t_{ij}$ for all $j>i$ \ref{algo:sort}. If no such arm is found, then it plays arm $K$.  $n_j(t), j\in [K] $ counts the total number of observation of pairs $(Y^i, Y^j)$, for all $i<j$, till round $t$ and uses it to update the estimates $\hgamma_{ij}^t$ (\ref{algo:Update}). Regret guarantees analogous to Thm~\ref{thm:ftregret} under SD for this new scheme can also be derived but requires additional conditions. Nevertheless, no such guarantee can be provided under WD because it is not uniformly learnable (Thm~3). 
\subsection{Extensions} 
We describe briefly a few extensions here, which we will elaborate on in a longer version of the paper.

\noindent
{\it Tree Structures:}
We can deal with trees, where now the sensors are organized as a tree with root node corresponding to sensor 1 and we can select any path starting from the root node. To see this, note that the marginal error condition of Eqn.~\ref{eqn:interp_opt} still characterizes optimal sensor selection. Under a modified variant of SD, namely, Eq.~\ref{eqn:DominanceCondition} is in terms of the children of an optimal sensor, it follows that it is again sufficient to keep track of disagreements. In an analogous fashion Eq.~\ref{eqn:WD} can be suitably modified as well. This leads to sublinear regret algorithm for tree-structures. 

\noindent
{\it Context-Dependent Sensor Selection:}
Each example has a context $x \in \R^d$ and a test maps (partial) context to the output space. Analogous to the context-independent case we impose context dependent notions for SD and WD, namely, Eq.~\ref{eqn:DominanceCondition} and Eq.~\ref{eqn:WD} hold conditioned on each $x \in {\cal X}$. To handle these cases we let $\gamma_i(x)\triangleq \Pr(Y^{i} \neq Y\mid x)$ and $\gamma_{ij}(x) \triangleq \Pr(Y^{i} \neq Y^{j}\mid x)$ denote the corresponding contextual error and disagreement probabilities. Our sublinear regret guarantees can be generalized for a parameterized GLM model for disagreement, namely, when the log-odds ratio $\log \frac{\gamma_{ij}(x)}{1-\gamma_{ij}(x)}$ can be written as $\theta_{ij}'x$ for some unknown $\theta_{ij} \in \mathbb{R}^d$.
\vspace{-5pt}

\section{Experiments}
\label{sec:Experiments}
In this section we evaluate performance of Algorithms \ref{alg:asym} and \ref{alg:UCB} on synthetic and real datasets (PIMA-Diabetes) and Heart Disease (Cleveland). Both of these datasets have accompanying costs for individual features.

{\bf Synthetic:} We generate data as follows. The input, $Y_t$, is generated IID Ber($0.7$). Outputs for channels 1, 2, 3 have an overall error statistic, $\gamma_1 = 0.4,\,\gamma_2=0.1,\,\gamma_3=0.05$ respectively. To ensure SD we enforce Defn.~\ref{eqn:DominanceCondition} during the generation process. To relax SD requirement we introduce errors upto 10\% during data generation for sensor outputs 2 and 3 when sensor 1 predicts correctly.


{\bf Real Datasets:} 
We split the features into three ``sensors'' based on their costs. For PIMA-Diabetes dataset (size=$768$) the first sensor is associated with patient history/profile with cost (\$6), the 2nd sensor in addition utilizes  insulin test (cost \$ 22) and the 3rd sensor uses all attributes (cost \$46). For the Heart dataset (size=$297$) we use the first $7$ attributes that includes cholestrol readings, blood-sugar, and rest-ECG (Cost \$27), the 2nd sensor utilizes in addition thalach, exang, oldpeak attributes that cost $\$300$  and the 3rd sensor utilizes more extensive tests at a total cost of \$601.

We train three linear SVM classifiers with 5-fold cross-validation and have verified that our results match known state-of-art. Note that Table~\ref{tab:ErrorTable1} shows that the resulting classifiers/tests on these datasets approximately satisfies SD condition and thus our setup should apply to these datasets. 
The size of these datasets is relatively small and limits our ability to experiment. To run the online algorithm we therefore generate an instance randomly from the dataset (with replacement) in each round. We repeat the experiments $20$ times and averages are shown with $95\%$ confidence bounds.

\begin{center}
\begin{figure*}[!bt]
\begin{minipage}{8cm}
		\centering
		\includegraphics[scale=0.45]{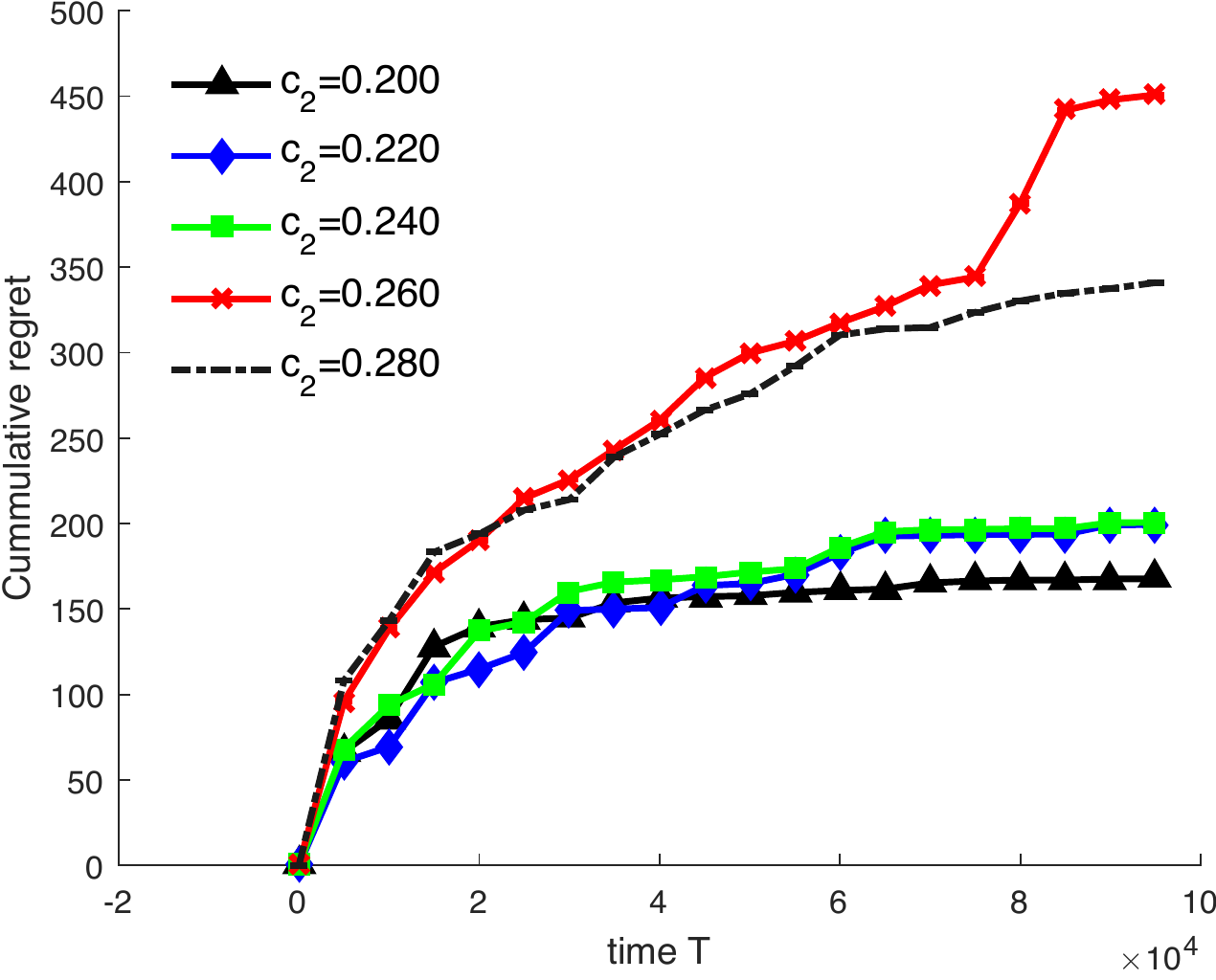}
				\vspace{-.3cm}
		\caption{\footnotesize Regret on Synthetic data under Strong Dominances}
\label{fig:BSC_SD}
	\end{minipage}
	\begin{minipage}{8cm}
		\centering
		\includegraphics[scale=0.45]{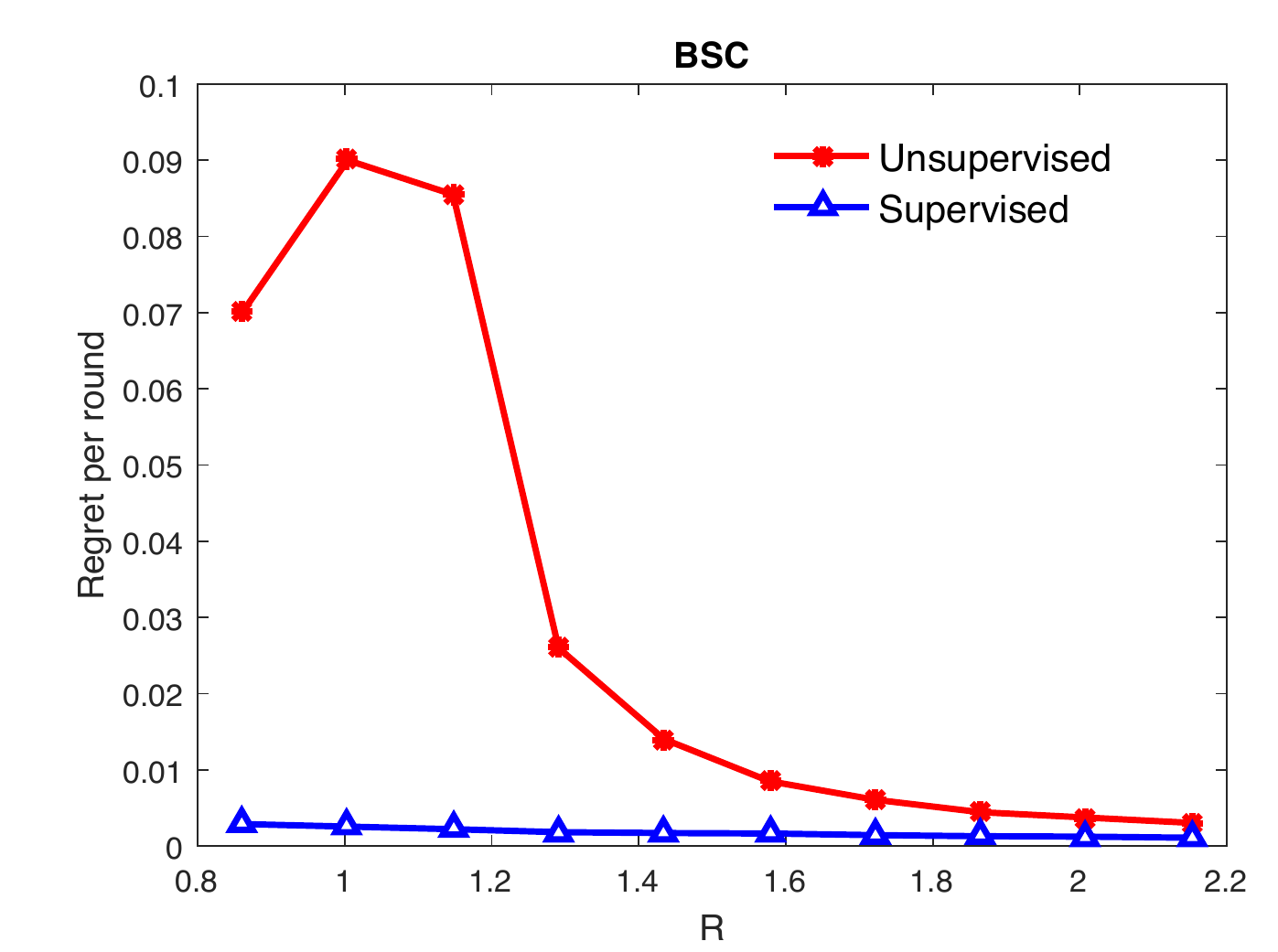}
				\vspace{-.3cm}
		\caption{\footnotesize Varying levels of Weak Dominance for Synthetic Data}
\label{fig:BSC_WD}
	\end{minipage}
%
\vspace{10pt}

\noindent
{\footnotesize Fig. \ref{fig:BSC_SD} depicts regret of USS (Alg.~1). Under SD it is always sublinear regardless of costs/probability. Fig. \ref{fig:BSC_WD} demonstrates phase-transition effect for USS (Alg.~2). Alg.~1 is not plotted here because it fails in this case. As $\rho \rightarrow 1$ (from the right) regret-per-round drastically increases thus validating our theory that WD is a maximal learnable set. 
}
\end{figure*}
\end{center}

\noindent
{\bf Testing Learnability:}
We experiment with different USS algorithms on the synthetic dataset. Our purpose is twofold: (a) verify sublinear regret reduction scheme (Alg~1) under SD (b) verify that WD condition is a maximal learnable set. 
Fig.~\ref{fig:BSC_SD} depicts the results of Alg.~\ref{alg:asym} when SD condition is satisfied and shows that we obtain sublinear regret regardless of costs/probabilities. To test WD requirement we parameterize the problem by varying costs. Without loss of generality we fix the cost of sensor 1 to be zero and the total cost of the entire system to be $C_{\mbox{tot}}$. We then vary cost of sensors 2 and 3. 
We test the hypothesis that WD is a maximal learnable set. We enforce Sensor 2 as the optimal sensor and vary the costs so that we continuously pass from the situation where WD holds ($\rho \geq 1$) to the case where WD does not hold ($\rho < 1$). 
Fig.~\ref{fig:BSC_WD} depicts regret-per-round for Alg.~\ref{alg:UCB} and as we can verify there is indeed a transition at $\rho = 1$ . 

\noindent
{\bf Unsupervised vs. Supervised Learning:}
The real datasets provide an opportunity to understand how different types of information can impact performance. We compare USS algorithm (Alg.~2) against a corresponding bandit algorithm where the learner receives feedback. In particular, for each action in each round, in the bandit setting, the learner knows whether or not the corresponding sensor output is correct. We implement the supervised bandit setting by replacing Step 5 in Alg.~2 with estimated marginal error rates. 

We scale costs by means of a tuning parameter (since the costs of features are all greater than one) and consider minimizing a combined objective ($\lambda$ Cost + Error) as stated in Sec.~2. High (low)-values for $\lambda$ correspond to low (high)-budget constraint. If we set a fixed budget of (\$ 50), this corresponds to high-budget (small $\lambda$) and low budget (large $\lambda$) for PIMA Diabetes (3rd test optimal) and Heart Disease (1st test optimal) respectively. Figs~\ref{fig:Diabetes} and \ref{fig:Heart} depicts performance for typical scenarios. We notice that for both high as well as low cost scenarios, while supervised does have lower regret, the USS cummulative regret is also sublinear and within a constant fraction of the supervised case. This is qualititively interesting because these plots demonstrate that (although under WD we do not have uniform learnability), in typical cases, we learn as well as the supervised setting. 

\begin{center}
\begin{figure*}
	\begin{minipage}{8cm}
		\centering
		\includegraphics[scale=0.5]{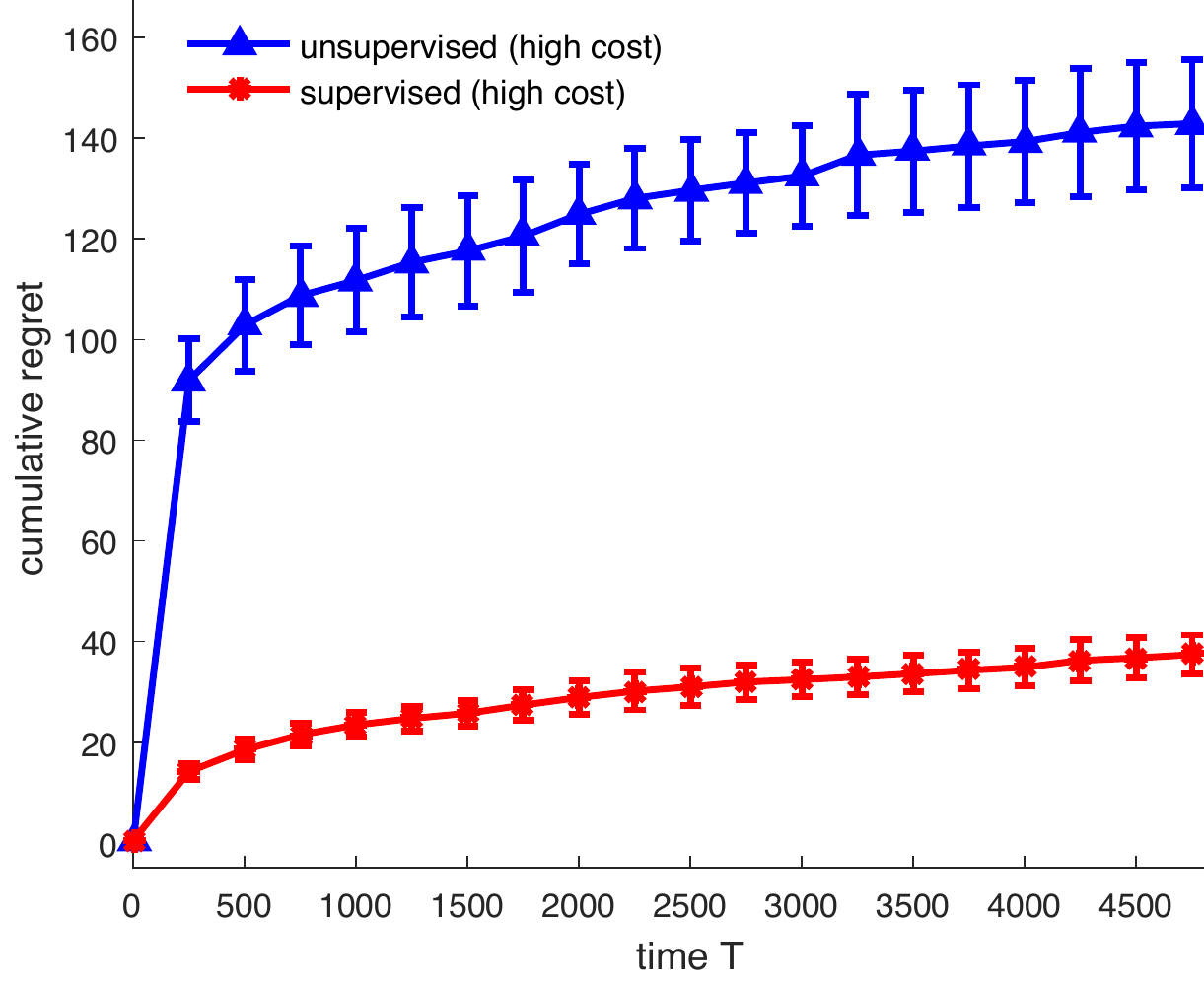}
		\vspace{-.3cm}
		\caption{\footnotesize Regret Curves on PIMA Diabetes}
		\label{fig:Diabetes}
	\end{minipage}
	\begin{minipage}{8cm}
		\centering
		\includegraphics[scale=0.5]{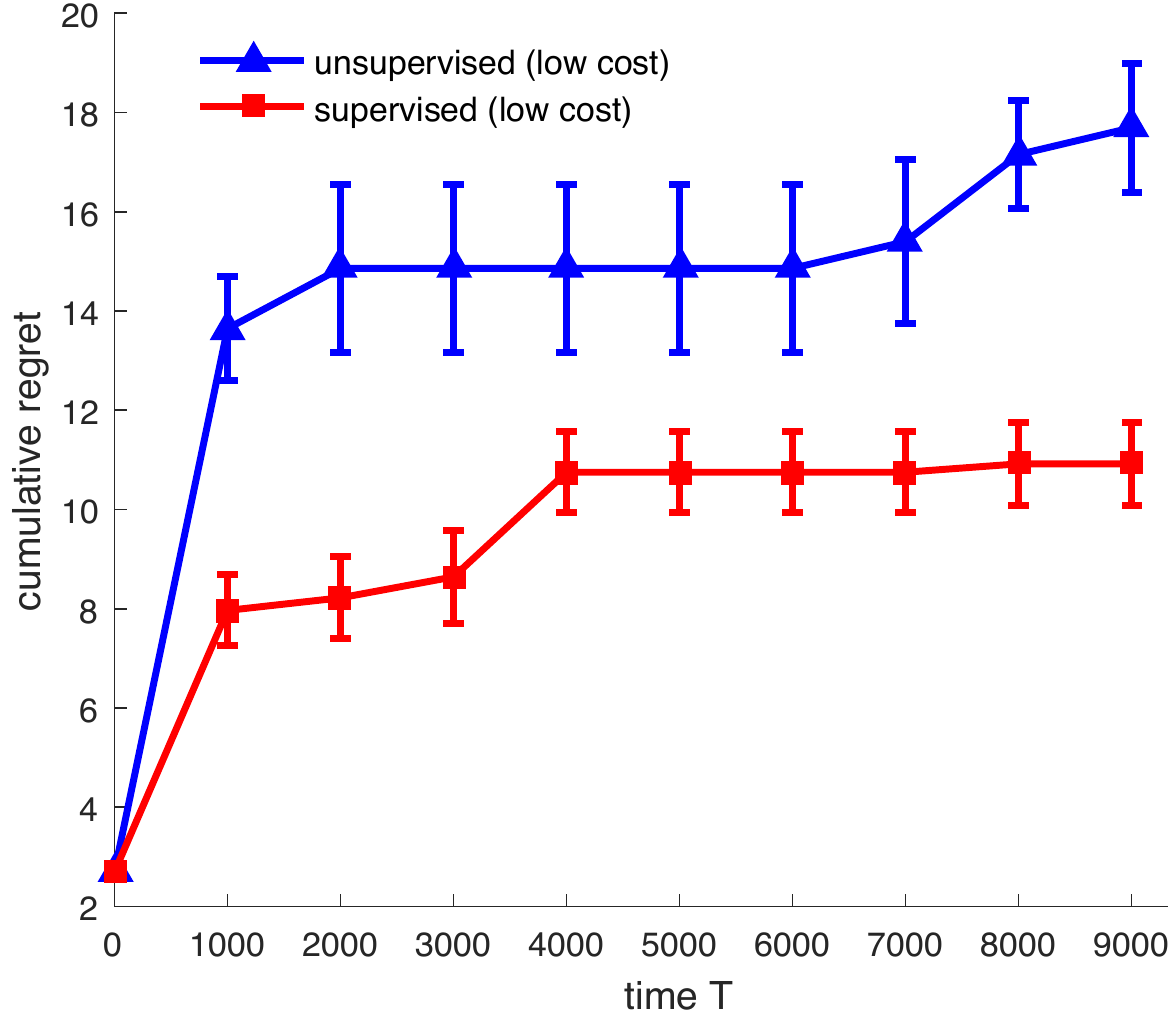}
		\vspace{-.3cm}
		\caption{Regret Curves on Heart Disease}
	\label{fig:Heart}
	\end{minipage}
\vspace{10pt}

\noindent
{\footnotesize Fig. \ref{fig:Diabetes} and Fig. \ref{fig:Heart} depict performance for real-datasets and presents comparisons against supervised (bandit with ground-truth feedback) scenario. Typically (but not in the worst-case (see Thm~5)), under WD, the unsupervised learning rate is within a constant factor of supervised setting.
}
\end{figure*}
\end{center}

\section{Conclusions}
\label{sec:Conclu}
The paper describes a novel approach for unsupervised sensor selection. These types of problems arise in a number of healthcare and security applications. In these applications we have a collection of tests that are typically organized in a hierarchical architecture wherein test results are acquired in a sequential fashion. The goal is to identify the best balance between accuracy \& test-costs. The main challenge in these applications is that ground-truth annotated examples are unavailable and it is often difficult to acquire them in-situ. We proposed a novel approach for sensor selection based on a novel notion of weak and strong dominance properties. We showed that weak dominance condition is maximal in that violation of this condition leads to loss of learnability. Our experiments demonstrate that weak dominance does hold in practice for real datasets and typically for these datasets, unsupervised selection can be as effective as a supervised (bandit) setting. 

\noindent
{\bf Acknowledgements:} The authors thank Dr. Joe Wang for helpful discussions and in particular suggesting the concept of strong dominance.

\section{Appendix}
\section*{Stochastic Partial Monitoring Problem \ref{sec:Setup}}
In an SPM a learner interacts with a stochastic environment in a sequential manner.
In round $t=1,2,\dots$ the learner chooses an action $A_t$ from an action set $\A$, and receives a feedback $Y_t\in \Y$
from a distribution $p$ which depends on the action chosen and also on the environment instance identified
with a ``parameter'' $\theta\in\Theta$:
$Y_t \sim p(\cdot;A_t,\theta)$. 
The learner also incurs a reward $R_t$, which is a function of the action chosen and the unknown parameter $\theta$:
$R_t = r(A_t,\theta)$. 
The reward may or may not be part of the feedback for round $t$.
The learner's goal is to maximize its total expected reward.
The family of distributions $(p(\cdot;a,\theta))_{a,\theta}$ and the family of rewards $(r(a,\theta))_{a,\theta}$
and the set of possible parameters $\Theta$ are known to the learner, who uses this knowledge to judiciously choose
its next action to reduce its uncertainty about $\theta$ so that it is able to eventually converge on choosing only an 
optimal action $a^*(\theta)$, achieving the best possible reward per round, $r^*(\theta) = \max_{a\in \A} r(a,\theta)$.  Bandit problems are a special case of SPMs where $\Y$ is the set of real numbers, $r(a,\theta)$ is the mean of distribution $p(\cdot;a,\theta)$.
\section*{Proof of Proposition \ref{prop:learnablemap}}
\begin{proof}
	$\Rightarrow$: Let $\Alg$ be an algorithm that achieves sublinear regret
	and pick  an instance  $\theta \in\Theta$. Let $P = P_S \otimes P_{Y|S}$.
	The regret $\Regret_n(\Alg,\theta)$ of $\Alg$ on instance $\theta$ can be written in the form
	\begin{align*}
	\Regret_n(\Alg,\theta) = \sum_{k\in [K]} \EEi{P_S}{N_k(n)} \Delta_k(\theta)\,,
	\end{align*}
	where $N_k(n)$ is the number of times action $k$ is chosen by $\Alg$ during the $n$ rounds while
	$\Alg$ interacts with $\theta$, $\Delta_k(\theta) = c(k,\theta) - c^*(\theta)$ is the immediate regret
	and $\EEi{P_S}{\cdot}$ denotes the expectation under the distribution induced by $P_S$.
	In particular, $N_k(n)$ hides dependence on the iid sequence $Y_1,\dots,Y_n \sim P_S$ 
	that we are taking the expectation over here. 
	Since the regret is sublinear, for any $k$ suboptimal action, $\EEi{P_S}{N_k(n)} = o(n)$. 
	Define $a(P_S) = \min \{ k\in [K]\,;\, \EEi{P_S}{N_k(n)} = \Omega(n) \,\}$. Then, $a$ is well-defined as the distribution of $N_k(n)$ for any $k$ depends only on $P_S$ (and $c$). Furthermore, $a(P_S)$ selects an optimal action.
	
	$\Leftarrow$: Let $a$ be the map in the statement and let $f:\N_+\to\N_+$ be such that $1\le f(n)\le n$ for any  $n\in \N$,
	$f(n)/\log(n) \to \infty$ as $n\to \infty$ and $f(n)/n \to 0$ as $n\to \infty$ (say, $f(n) = \lceil \sqrt{n} \rceil$).
	Consider the algorithm that chooses $I_t = K$ for the first $f(n)$ steps, after which it estimates $\hat{P}_S$ by
	frequency counting and then uses $I_t = a(\hat{P}_S)$ in the remaining $n-f(n)$ trials. 
	Pick any $\theta \in \Theta$ so that $\theta = P_S \otimes P_{Y|S}$. 
	Note that by Hoeffding's inequality, 
	$\sup_{y\in \{0,1\}^K} |\hat{P}_S(y)  - P_S(y)| \le \sqrt{\frac{K\log(4n)}{2f(n)}}$ holds with probability $1-1/n$.
	Let $n_0$ be the first index such that for any $n\ge n_0$,
	$\sqrt{\frac{K\log(4n)}{2f(n)}}\le \Delta^*(\theta) \doteq \min_{k:\Delta_k(\theta)>0} \Delta_k(\theta)$.
	Such an index $n_0$ exists by our assumptions that $f$ grows faster than $n \mapsto \log(n)$.
	For $n\ge n_0$, the expected regret of $\Alg$ is at most $n \times 1/n + f(n) (1-1/n) \le 1+f(n) = o(n)$.
	\end{proof}

\section*{Proof of Corollary \ref{cor:twdlearnable}}
\begin{proof}
	By \cref{prop:awdwelldef}, $\awd$ is well-defined over $\TWD$, while by \cref{prop:awdsound}, $\awd$ is sound over $\TWD$.
\end{proof}

\section*{Proof of Proposition \ref{prop:gammadiff}}

\begin{proof}
	
	We construct a map as required by~\cref{prop:learnablemap}.
	Take an instance $\theta \in \TWD$ and let $\theta = P_S \otimes P_{Y|S}$ be its decomposition
	as defined above.
	Let $\gamma_i = \Prob{Y^i \ne Y}$, $(Y,Y^1,\dots,Y^K)\sim \theta$.
	For identifying an optimal action in $\theta$, it clearly suffices
	to know the sign of $\gamma_i + C_i - (\gamma_j +C_j)$ for all pairs $i,j\in [K]^2$.
	Since $C_i - C_j$ is known, it remains to study $\gamma_i-\gamma_j$.
	Without loss of generality (WLOG) let $i<j$.
	Then, 
	\begin{align*}
	\MoveEqLeft 0  \le \gamma_i  - \gamma_j = \Prob{Y^i\ne Y} - \Prob{Y^j\ne Y} \\
	& = \cancel{\Prob{Y^i\ne Y, Y^i=Y^j}} + \Prob{ Y^i\ne Y, Y^i\ne Y^j } - \\
	& - \left\{ 
	\cancel{\Prob{Y^j\ne Y, Y^i=Y^j}} + \Prob{ Y^j\ne Y, Y^i\ne Y^j }\right\}\\
	& = \Prob{ Y^i\ne Y, Y^i \ne Y^j } + \Prob{Y^i=Y,Y^i\ne Y^j}       \\
	& - \left\{ 
	\Prob{ Y^j \ne Y, Y^i\ne Y^j } + \Prob{ Y^i=Y,Y^i\ne Y^j}
	\right\}\\
	& \stackrel{\footnotesize (a)}{=} \Prob{ Y^j \ne Y^i } -2 \Prob{ Y^j\ne Y, Y^i = Y }\,,
	\end{align*}
	where in $(a)$ we used that $\Prob{ Y^j \ne Y, Y^i\ne Y^j } =  \Prob{ Y^j\ne Y,Y^i= Y}$ and also
	$\Prob{ Y^i=Y,Y^i\ne Y^j} = \Prob{ Y^j\ne Y,Y^i= Y}$
	which hold because $Y,Y^i,Y^j$ only take on two possible values.
\end{proof}

\section*{Proof of Proposition \ref{prop:ilej}}
\begin{proof}
	\noindent $\Rightarrow$: From the premise, it follows that $C_j - C_i \ge \gamma_i - \gamma_j$.
	Thus, by~\eqref{eq:cond1}, $C_j - C_i \ge \Prob{Y^i\ne Y^j}$.
	\noindent $\Leftarrow$: We have $C_j - C_i \ge \Prob{Y^i \ne Y^j} \ge \gamma_i -\gamma_j$, where the last
	inequality is by \cref{cor:gammadiff}.
	\end{proof}
\section*{Proof of proposition \ref{prop:jlei}}	
	\begin{proof}
		\noindent $\Rightarrow$: The condition $\gamma_i + C_i \le \gamma_j + C_j$ implies that $\gamma_j -\gamma_i \ge C_i - C_j$.
		By \cref{cor:gammadiff} we get $\Prob{Y^i \ne Y^j} \ge C_i - C_j$.
		\noindent $\Leftarrow$: Let $C_i - C_j \le \Prob{Y^i \ne Y^j}$. Then, by \eqref{eq:cond2}, $C_i - C_j \le \gamma_j - \gamma_i$.
	\end{proof}

\section*{Proof of Theorem \ref{thm:MaxLearnability}}
\begin{proof}
	Let $a$ as in the theorem statement. By~\cref{prop:awdunique}, $\awd$ is the unique sound action-selector map over $\TWD$.
	Thus, for any $\theta = P_S\otimes P_{Y|S}\in \TWD$, $\awd(P_S) = a(P_S)$.
	Hence, the result follows from \cref{cor:awdoutsideincorrect}.
	\end{proof}

\section*{Proof of Proposition \ref{prop:awdwelldef}}
\begin{proof}
	Let $\theta\in \TWD$, $i = a^*(\theta)$. Obviously, \eqref{eq:wd2} holds by the definition of $\TWD$.
	Thus, the only question is whether \eqref{eq:wd1} also holds.
	We prove this by contradiction:
	Thus, assume that \eqref{eq:wd1} does not hold, i.e., for some $j<i$, $C_i-C_j \ge \Prob{Y^i \ne Y^j}$. Then, by \cref{cor:gammadiff}, $\Prob{ Y^i \ne Y^j} \ge \gamma_j - \gamma_i$, hence $\gamma_j + C_j \le \gamma_i + C_i$, which contradicts the definition of $i$, thus finishing the proof.
	\end{proof}
\section*{Proof of Proposition \ref{prop:awdsound}}
\begin{proof}
	Take any $\theta\in \TWD$ with $\theta = P_S\otimes P_{Y|S}$, $i = \awd(P_S)$, $j = a^*(\theta)$.
	If $i=j$, there is nothing to be proven. Hence, first assume that $j>i$. Then, by~\eqref{eq:wd2}, $C_j - C_i \ge \Prob{Y^i \ne Y^j}$.
	By \cref{cor:gammadiff}, $\Prob{Y^i \ne Y^j } \ge \gamma_i - \gamma_j$. Combining these two inequalities we get that
	$\gamma_i + C_i \le \gamma_j + C_j$, which contradicts with the definition of $j$.
	Now, assume that $j<i$. Then, by~\eqref{eq:wd}, $C_i - C_j \ge \Prob{Y^i \ne Y^j}$.
	However, by~\eqref{eq:wd1}, $C_i -C_j < \Prob{Y^i \ne Y^j}$, thus $j<i$ cannot hold either and we must have $i=j$.
	\end{proof}

\section*{Proof of Proposition \ref{prop:awdunique}}
\begin{proof}
	Pick any $\theta = P_S\otimes P_{Y|S}\in \TWD$. If $A^*(\theta)$ is a singleton, then clearly $a(P_S) = \awd(P_S)$ since both are sound over $\TWD$.
	Hence, assume that $A^*(\theta)$ is not a singleton.
	Let $i = a^*(\theta) = \min A^*(\theta)$ and let $j = \min A^*(\theta) \setminus \{ i \}$.
	We argue that $P_{Y|S}$ can be changed so that on the new instance $i$ is still an optimal action, while
	$j$ is not an optimal action, while the new instance $\theta' = P_S \otimes P_{Y|S}'$ is in $\TWD$.
	
	The modification is as follows:
	Consider any $y^{-j} \doteq (y^1,\dots,y^{j-1},y^{j+1},\dots,y^K)\in \{0,1\}^{K-1}$.
	For $y,y^j\in \{0,1\}$, define 
	$q(y|y^j) = P_{Y|S}(y|y^1, \dots, y^{j-1}, y^j, y^{j+1},\dots, y^K)$
	and similarly let
	$q'(y|y^j) = P_{Y|S}'(y|y^1, \dots, y^{j-1}, y^j, y^{j+1},\dots, y^K)$
	Then, we let $q'(0|0) = 0$ and $q'(0|1) = q(0|0) + q(0|1)$,
	while we let  $q'(1|1) = 0$ and $q'(1|0) = q(1|1) + q(1|0)$.
	This makes $P_{Y|S}'$ well-defined ($P_{Y|S}'(\cdot|y^1,\dots,y^K)$ is a distribution for any $y^1,\dots,y^K$).
	Further, we claim that the transformation has the property that 
	it leaves $\gamma_p$ unchanged for $p\ne j$, while $\gamma_j$ is guaranteed to decrease.
	To see why $\gamma_p$ is left unchanged for $p\ne j$ note that
	$\gamma_p = \sum_{y^p}  P_{Y^p}(y^p) P_{Y|Y^p}(1-y^p|y^p)$.
	Clearly, $P_{Y^p}$ is left unchanged.
	Introducing $y^{-k}$ to denote a tuple where the $k$th component is left out,
	$P_{Y|Y^p}(1-y^p|y^p) = \sum_{y^{-p,-j}} P_{Y|Y^1,\dots,Y^K}( 1-y^p | y^1,\dots, y^{j-1}, 0, y^{j+1}, \dots, y^K )
	+P_{Y|Y^1,\dots,Y^K}( 1-y^p | y^1,\dots, y^{j-1}, 1, y^{j+1}, \dots, y^K )$
	and by definition,
	\begin{align*}
	& P_{Y|Y^1,\dots,Y^K}( 1-y^p | y^1,\dots, y^{j-1}, 0, y^{j+1}, \dots, y^K )\\
	&\quad +P_{Y|Y^1,\dots,Y^K}( 1-y^p | y^1,\dots, y^{j-1}, 1, y^{j+1}, \dots, y^K )\\
	&
	=
	P_{Y|Y^1,\dots,Y^K}'( 1-y^p | y^1,\dots, y^{j-1}, 0, y^{j+1}, \dots, y^K )\\
	&\quad+P_{Y|Y^1,\dots,Y^K}'( 1-y^p | y^1,\dots, y^{j-1}, 1, y^{j+1}, \dots, y^K )\,,
	\end{align*}
	where the equality holds because ``$q'(y|0)+q'(y|1) = q(y|0) + q(y|1)$''.
	Thus, $P_{Y|Y^p}(1-y^p|y^p) = P_{Y|Y^p}'(1-y^p|y^p)$ as claimed.
	That $\gamma_j$ is non-increasing follows with an analogue calculation.
	In fact, this shows that $\gamma_j$ is strictly decreased
	if for any $(y^1,\dots,y^{j-1},y^{j+1},\dots,y^K)\in \{0,1\}^{K-1}$, either $q(0|0)$ or $q(1|1)$ was positive.
	If these are never positive, this means that $\gamma_j=1$. 
	But then $j$ cannot be optimal since $c_j>0$.
	Since $j$ was optimal, $\gamma_j$ is guaranteed to decrease.
	
	Finally, it is clear that the new instance is still in $\TWD$ since  $a^*(\theta)$ is left unchanged.
\end{proof}

\section*{Proof of  Theorem \ref{thm:tsdlearnable}}
\begin{proof}[Proof of \cref{thm:tsdlearnable}]
	We construct a map as required by~\cref{prop:learnablemap}.
	Take an instance $\theta \in \TSD$ and let $\theta = P_S \otimes P_{Y|S}$ be its decomposition as before.
	Let $\gamma_i = \Prob{Y^i \ne Y}$, $(Y,Y^1,\dots,Y^K)\sim \theta$, $C_i = c_1+\dots+c_i$.
	For identifying an optimal action in $\theta$, it clearly suffices
	to know the sign of $\gamma_i + C_i - (\gamma_j +C_j) = \gamma_i-\gamma_j + (C_i-C_j)$ for all pairs $i,j\in [K]^2$.
	Without loss of generality (WLOG) let $i<j$. By \cref{prop:gammadiff},
	$\gamma_i - \gamma_j = \Prob{ Y^i \ne Y^j } -2 \Prob{ Y^j\ne Y, Y^i = Y }$.
	Now, since $\theta$ satisfies the strong dominance condition, $ \Prob{ Y^j\ne Y, Y^i = Y } = 0$.
	Thus, $\gamma_i - \gamma_j = \Prob{ Y^i \ne Y^j }$
	which is a function of $P_S$ only.
	Since $(C_i)_i$ are known, a map as required by~\cref{prop:learnablemap} exists.
\end{proof}

\section*{Proof of Theorem \ref{thm:nonunif}}
\begin{proof}
We first consider the case when $K=2$ and arbitrarily choose $C_2 - C_1 = 1/4$. 
We will consider two instances, $\theta,\theta'\in \TWD$ such that for instance $\theta$, 
action $k=1$ is optimal with an action gap of $c(2,\theta) - c(1,\theta) = 1/4$ between the cost of the second and the first
action,  while for instance $\theta'$, $k=2$ is the optimal action and the action gap is $c(1,\theta) - c(2,\theta) = \epsilon$
where $0<\epsilon<3/8$.
Further, the entries in $P_S(\theta)$ and $P_S(\theta')$ differ by at most $\epsilon$. 
From this, a standard reasoning gives that no algorithm can achieve sublinear minimax regret over $\TWD$ because any
algorithm is only able to identify $P_S$. 

The constructions of $\theta$ and $\theta'$ are shown in \cref{tab:nonunif}:
The entry in a cell gives the probability of the event as specified by the column and row labels.
For example, in instance $\theta$, $3/8$ is the probability of $Y=Y^1=Y^2$, 
while the probability of $Y^1=Y\ne Y^2$ is $1/8$. Note that the cells with zero actually 
correspond to impossible events, i.e., these cannot be assigned a positive probability.
The rationale of a redundant (and hence sparse) table is so that probabilities of certain events of interest, such as $Y^1\ne Y^2$ are easier to determine based on the table. The reader should also verify that the positive probabilities correspond to events that are possible.
\bgroup
\def\arraystretch{1.5}
\begin{table}[]
	\centering
	\begin{tabular}{|c|c|c|c|}
		\hline
		\multicolumn{2}{|c|}{Instance $\theta$}  & $Y^1=Y^2$     & $Y^1\ne Y^2$  \\ \hline
		\multirow{2}{*}{$Y^1= Y$}   & $Y^2= Y$   & $\frac{3}{8}$ & $0$           \\ \cline{2-4} 
		& $Y^2\ne Y$ & $0$ & $\frac{1}{8}$ \\ \hline
		\multirow{2}{*}{$Y^1\ne Y$} & $Y^2= Y$   & $0$ & $\frac{1}{8}$           \\ \cline{2-4} 
		& $Y^2\ne Y$ & $\frac{3}{8}$ & $0$ \\ \hline
	\end{tabular}
	\begin{tabular}{|c|c|c|c|}
		\hline
		\multicolumn{2}{|c|}{Instance $\theta'$}  & $Y^1=Y^2$     & $Y^1\ne Y^2$  \\ \hline
		\multirow{2}{*}{$Y^1= Y$}   & $Y^2= Y$   & $\frac{3}{8}-\epsilon$ & $0$           \\ \cline{2-4} 
		& $Y^2\ne Y$ & $0$ & $0$ \\ \hline
		\multirow{2}{*}{$Y^1\ne Y$} & $Y^2= Y$   & $0$ & $\frac{2}{8}+\epsilon$           \\ \cline{2-4} 
		& $Y^2\ne Y$ & $\frac{3}{8}$ & $0$ \\ \hline
	\end{tabular}
	\vspace*{0.1in}
	\caption{The construction of two problem instances for the proof of \cref{thm:nonunif}.}
	\label{tab:nonunif}
\end{table}
\egroup

We need to verify the following:
{\em (i)} $\theta,\theta'\in \TWD$;
{\em (ii)} the optimality of the respective actions in the respective instances;
{\em (iii)} the claim concerning the size of the action gaps;
{\em (iv)} that $P_S(\theta)$ and $P_S(\theta')$ are close.
Details of the calculations to support {\em (i)}--{\em (iii)} can be found in \cref{tab:nonunif2}.
The row marked by $(*)$ supports that the instances are proper USS instances.
In the row marked by $(**)$, there is no requirement for $\theta'$ because 
in $\theta'$ action two is optimal, and hence there is no action with larger index 
than the optimal action, hence $\theta'\in \TWD$ automatically holds.
To verify the closeness of $P_S(\theta)$ and $P_S(\theta')$ we actually 
would need to first specify $P_S$ (the tables do not fully specify these).
However, it is clear the only restriction we put on $P_S$ is the value of $\Prob{Y^1\ne Y^2}$ (and
that of $\Prob{Y^1=Y^2}$) and these values are within an $\epsilon$ distance of each other.
Hence, $P_S$ can also be specified to satisfy this. In particular, one possibility for $P$ and $P_S$ are given in \cref{tab:nonunif3}.
\bgroup
\def\arraystretch{1.5}
\begin{table}[]
	\centering
	\begin{tabular}{|c|c|c|}
		\hline
		& $\theta$                & $\theta'$ \\ \hline
		$\gamma_1 = \Prob{Y^1\ne Y}$ & $\frac{1}{4}$           & $\frac{5}{8}+\epsilon$ \\ \hline
		$\gamma_2 = \Prob{Y^2\ne Y}$ & $\frac{1}{4}$           & $\frac{3}{8}$ \\ \hline
		$\gamma_2 \le \gamma_1 \mbox{}^{(*)}$        & \checkmark           & \checkmark \\ \hline
		$c(1,\cdot)$                                 & $\frac{1}{4}$           & $\frac{5}{8}+\epsilon$ \\ \hline
		$c(2,\cdot)$                                 & $\frac{2}{4}$           & $\frac{5}{8}$ \\ \hline
		$a^*(\cdot)$                                 & $k=1$                   & $k=2$ \\ \hline
		$\Prob{Y^1\ne Y^2}$                   & $\frac{1}{4}$         & $\frac{1}{4}+\epsilon$ \\ \hline
		$\theta \in \TWD  \mbox{}^{(**)}$                        & $\frac{1}{4}\ge \frac14$ \checkmark & \checkmark \\ \hline
		$|c(1,\cdot)-\c(2,\cdot)|$              & $\frac{1}{4}$         & $\epsilon$ \\ \hline
	\end{tabular}
	\vspace*{0.1in}
	\caption{Calculations for the proof of \cref{thm:nonunif}.}
	\label{tab:nonunif2}
\end{table}
\egroup

\bgroup
\def\arraystretch{1.5}
\begin{table}[]
	\centering
	\begin{tabular}{|c|c|c||c|c|}
		\hline
		$Y^1$ & $Y^2$ & $Y$ & $\theta$ & $\theta'$ \\ \hline\hline
		$0$ & $0$ & $0$ & $\frac38$ & $\frac38-\epsilon$ \\ \hline
		$0$ & $0$ & $1$ & $\frac38$ & $\frac38-\epsilon$ \\ \hline
		$0$ & $1$ & $0$ & $0         $ & $0                       $ \\ \hline
		$0$ & $1$ & $1$ & $0         $ & $0                       $ \\ \hline
		$1$ & $0$ & $0$ & $\frac18$ & $\frac28+\epsilon$ \\ \hline
		$1$ & $0$ & $1$ & $\frac18$ & $0                       $ \\ \hline
		$1$ & $1$ & $0$ & $0         $ & $0                       $ \\ \hline
		$1$ & $1$ & $1$ & $0         $ & $0                       $ \\ \hline
	\end{tabular}
	\mbox{}
	\vspace*{2in}
	\hspace*{0.5in}
	\mbox{}
	\begin{tabular}{|c|c||c|c|}
		\hline
		$Y^1$ & $Y^2$ &$\theta$ & $\theta'$ \\ \hline\hline
		$0$ & $0$ & $\frac68$ & $\frac68-\epsilon$ \\ \hline
		$0$ & $1$ & $0         $ & $0                       $ \\ \hline
		$1$ & $0$ & $\frac28$ & $\frac28+\epsilon$ \\ \hline
		$1$ & $1$ & $0         $ & $0                       $ \\ \hline
	\end{tabular}
	
	\vspace*{0.1in}
	\caption{Probability distributions for instances $\theta$ and $\theta'$. On the left are shown the joint
		probability distributions, while on the right are shown their marginals
		for the sensors.}
	\label{tab:nonunif3}
\end{table}
\egroup

\end{proof}

\section*{Proof of Proposition \ref{prop:equivalence}}
\begin{proof}First note that the mapping of the policies is such that number of pull of arm $k$ after $n$ rounds by policy $\pi$ on problem instance $f(\theta)$ is the same as the number of pulls of arm $k$ by $\pi^\prime$ on problem instance $\theta$. Recall that mean value of arm $k$ in problem instance $\theta$ $ is \gamma_k +C_k$ and that of corresponding arm in problem instance $f(\theta)$ is $\gamma_1-(\gamma_i+C_i)$. We have
	
	\begin{align*}
	\Regret_n(\pi^\prime,\theta) = \sum_{k\in [K]} \EEi{P_S}{N_k(n)}(\gamma_k+C_k-\gamma_{k^*}-C_{k^*}) \,,
	\end{align*}
	and
	\begin{align*}
	&\Regret_n(\pi,f(\theta))\\
	&= \sum_{k\in [K]} \EEi{P_S}{N_k(n)}\left (\max_{i \in [K]}\{\gamma_1-\gamma_i- C_i\}-(\gamma_1-\gamma_k- C_k)\right) \\
	&= \sum_{k\in [K]} \EEi{P_S}{N_k(n)}\left (\gamma_k+ C_k - \min_{i \in [K]}\{\gamma_i+ C_i\}\right) \\
	&=\Regret_n(\pi^\prime,\theta).
	\end{align*}
	\end{proof}

\end{document}